\documentclass[final,12pt]{colt2021}

\usepackage{hyperref}
\usepackage{url}
\usepackage{times}
\usepackage{natbib}

\usepackage{mathtools}
\usepackage{enumerate}
\usepackage{thm-restate}

\usepackage[mathscr]{euscript}
\usepackage{makecell}
\usepackage{xcolor, colortbl}
\usepackage{prettyref}
\usepackage{wrapfig}

\usepackage{MnSymbol}
\def\Rset{\mathbb{R}}

\let\Pr\undefined

\DeclareMathOperator*{\Pr}{\mathbb{P}}

\DeclareMathOperator*{\E}{\mathbb E}

\DeclareMathOperator*{\argmin}{argmin}

\DeclareMathOperator{\Tr}{Tr}

\DeclarePairedDelimiter{\abs}{\lvert}{\rvert} 
\DeclarePairedDelimiter{\bracket}{[}{]}
\DeclarePairedDelimiter{\curl}{\{}{\}}
\DeclarePairedDelimiter{\norm}{\|}{\|}
\DeclarePairedDelimiter{\paren}{(}{)}
\DeclarePairedDelimiter{\tri}{\langle}{\rangle}

\newcommand{\cL}{\mathcal{L}}

\newcommand{\cN}{\mathcal{N}}

\newcommand{\sH}{{\mathscr H}}

\newcommand{\sP}{{\mathscr P}}
\newcommand{\sQ}{{\mathscr Q}}

\newcommand{\sU}{{\mathscr U}}
\newcommand{\sX}{{\mathscr X}}
\newcommand{\sY}{{\mathscr Y}}

\newcommand{\bM}{{\mathbf M}}

\newcommand{\bu}{{\mathbf u}}

\newcommand{\bw}{{\mathbf w}}
\newcommand{\bPhi}{{\mathbf \Phi}}

\newcommand{\sfp}{{\mathsf p}}
\newcommand{\sfq}{{\mathsf q}}

\newcommand{\sfu}{{\mathsf u}}
\newcommand{\sfv}{{\mathsf v}}

\newcommand{\sfB}{{\mathsf B}}

\newcommand{\balpha}{\boldsymbol{\alpha}}

\newcommand{\disc}{\mathrm{dis}}
\newcommand{\dis}{\mathrm{dis}}
\newcommand{\Dis}{\mathrm{Dis}}
\newcommand{\Rad}{\mathfrak R}
\newcommand{\bsigma}{{\boldsymbol \sigma}}

\newcommand{\best}{\text{\sc best}}
\newcommand{\sbest}{\text{\sc sbest}}
\newcommand{\bda}{\text{\sc best-da}}

\newcommand{\dm}{\text{\sc dm}}

\newcommand{\sbda}{\text{\sc sbest-da}}

\newcommand{\h}{\widehat}
\newcommand{\ov}{\overline}

\newcommand{\e}{\epsilon}

\newcommand{\ignore}[1]{}

\hypersetup{
  colorlinks   = true,
  urlcolor     = blue,
  linkcolor    = blue,
  citecolor    = blue
}

\ignore{
\theoremstyle{thmstyleone}%

\theoremstyle{thmstyletwo}%
\theoremstyle{thmstylethree}%
}

\usepackage[toc,page,header]{appendix}
\begin{document}

\title[Best-Effort Adaptation]{Best-Effort Adaptation}

\coltauthor{%
 \Name{Pranjal Awasthi} \Email{pranjalawasthi@google.com}\\
 \addr Google Research, Mountain View
 \AND
 \Name{Corinna Cortes} \Email{corinna@google.com}\\
 \addr Google Research, New York%
 \AND
 \Name{Mehryar Mohri} \Email{mohri@google.com}\\
 \addr Google Research and Courant Institute of Mathematical Sciences, New York%
}

\maketitle

\begin{abstract}%

  We study a problem of \emph{best-effort adaptation} motivated by
  several applications and \ignore{fairness }considerations, which
  consists of determining an accurate predictor for a target domain,
  for which a moderate amount of labeled samples are available, while
  leveraging information from another domain for which substantially
  more labeled samples are at one's disposal. We present a new and
  general discrepancy-based theoretical analysis of sample reweighting
  methods, including bounds holding uniformly over the
  weights. We show how these bounds can guide the design of learning
  algorithms that we discuss in detail. We further show that our
  learning guarantees and algorithms provide improved solutions for
  standard domain adaptation problems, for which few labeled data or
  none are available from the target domain.  We finally report the
  results of a series of experiments demonstrating the effectiveness
  of our best-effort adaptation and domain adaptation algorithms, as
  well as comparisons with several baselines. We also discuss how our
  analysis can benefit the design of principled solutions for
  \emph{fine-tuning}.
\end{abstract}


\section{Introduction}

Consider the following adaptation problem that frequently arises in
applications.  Suppose we have access to a fair amount of labeled data
from a target domain $\sP$ and to a significantly larger amount of
labeled data from a different domain $\sQ$. How can we best exploit
both collections of labeled data to come up with as accurate a
predictor as possible for the target domain $\sP$? We will refer to
this problem as the \emph{best-effort adaptation problem} since we
seek the best method to leverage the additional labeled data from
$\sQ$ to come up with a best predictor for $\sP$.  One would imagine
that the data from $\sQ$ should be helpful in improving upon the
performance obtained by training only on the $\sP$ data, if $\sQ$ is
not too different from $\sP$. The question is how to measure this
difference and account for it in the learning algorithm. This
best-effort problem differs from standard domain adaptation problems
where typically very few or no labeled data from the target is at
one's disposal.

Best-effort adaptation can also be motivated by fairness
considerations, such as racial disparities in automated speech
recognition
\citep{KoeneckeNamLakeNudellQuarteyMengeshaToupsRickfordJurafskyGoel2020}.
A significant gap has been reported for the accuracy of speech
recognition systems when tested on speakers of vernacular English
versus non-vernacular English speakers. In practice, there is a
substantially larger amount of labeled data available for the
non-vernacular domain since it represents a larger population of
English speakers. As a result, it might not be possible, with the
training data in hand, to achieve an accuracy for vernacular speech
similar to the one achieved for non-vernacular speech.  Such a
recognition system might therefore have only one method for equalizing
accuracy between these populations: namely, degrading the system's
performance on the larger population.
Alternatively, one could instead formulate the problem of maximizing
the performance of the system on the vernacular speakers, leveraging
{\em all} the data available at hand to find the \emph{best-effort}
predictor for vernacular speakers.

Here, we present a detailed study of best-effort adaptation, including
a new and general theoretical analysis of reweighting methods using
the notion of discrepancy, as well as new algorithms and empirical
evaluations. We further show how our analysis can be extended to that
of domain adaptation problems, for which we also design new algorithms
and report experimental results.

There is a very broad literature dealing with adaptation solutions for
distinct scenarios and we cannot present a comprehensive survey here.
Instead, we briefly discuss here the most closely related work and
give a detailed discussion of previous work in
Appendix~\ref{app:related}. We also refer the reader to papers such as
\citep{pan2009,wang2018deep}.
Let us add that similar scenarios to best-effort adaptation have
been studied in the past under some different names such as
\emph{inductive transfer} or \emph{supervised domain adaptation} but
with the assumption of much smaller labeled data from the target
domain \citep{GarckeVanck2014,HedegaardSheikhOmarIIosifidis2021}.

The work we present includes a significant theoretical component and
benefits from prior theoretical analyses of domain adaptation. The
theoretical analysis of domain adaptation was initiated by
\cite{KiferBenDavidGehrke2004} and
\cite{BenDavidBlitzerCrammerPereira2006} with the introduction of a
\emph{$d_A$-distance} between distributions. They used this notion to
derive VC-dimension learning bounds for the zero-one loss, which was
elaborated on in subsequent works
\citep{BlitzerCrammerKuleszaPereiraWortman2008,
  BenDavidBlitzerCrammerKuleszaPereiraVaughan2010}.  Later,
\cite{MansourMohriRostamizadeh2009} and
\cite{CortesMohri2011,CortesMohri2014} presented a general analysis
of single-source adaptation for arbitrary loss functions, where they
introduced the notion of \emph{discrepancy}, a divergence measure
nicely aligned with domain adaptation.  Discrepancy coincides with the
$d_A$-distance in the special case of the zero-one loss. It takes into
account the loss function and hypothesis set and, importantly, can be
estimated from finite samples. The authors gave a discrepancy
minimization algorithm based on a reweighting of the losses of sample
points.  We use their notion of discrepancy in our new analysis.
\cite{CortesMohriMunozMedina2019} presented an extension of the
discrepancy minimization algorithm based on the so-called
\emph{generalized discrepancy}, which both incorporates a
hypothesis-dependency and  works with a less conservative notion
of \emph{local discrepancy} defined by a supremum over a subset of the
hypothesis set. The notion of local discrepancy has been since adopted
in several recent publications, in the study of active learning or
adaptation \citep{DeMathelinMougeotVayatis2021,ZhangLiuLongJordan2019,
  ZhangLongWangJordan2020} and is also used in part of our analysis.

While our main motivation is best-effort adaptation, in
Section~\ref{sec:disc-theory}, we present a general analysis that holds for
\emph{all sample reweighting methods}. Our theoretical
analysis and learning bounds are new and are based on the notion of
discrepancy.  They include learning guarantees holding uniformly with
respect to the weights, as well as a lower bound suggesting the
importance of the discrepancy term in our bounds. Our theory guides
the design of principled learning algorithms for best-effort
adaptation, \best\ and \sbest, that we discuss in detail in
Section~\ref{sec:disc-algorithms}. This includes our estimation of the
discrepancy terms via DC-programming (Appendix~\ref{app:disc-est}).

In Section~\ref{sec:da}, we further show how our analysis can be
extended to the case where few labeled data or none are available from
the target domain, that is the scenario of (unsupervised or weakly
supervised) domain adaptation. Here too, we derive new
discrepancy-based learning bounds based on reweighting, including
uniform bounds with respect to the weights
(Section~\ref{sec:da-bounds}). Interestingly, here, an additional set
of sample weights naturally appears in the analysis, to account for
the absence of labels from the target. Our theoretical analysis leads
to the design of a new adaptation algorithms,
\bda\ (Section~\ref{sec:da-algorithm}). We further discuss in detail
how in this scenario labeled discrepancy terms can be upper-bounded in
terms of unlabeled ones, including unlabeled local discrepancies, and
how some additional amount of labeled data can be beneficial
(Section~\ref{app:disc-upper}).

In Section~\ref{sec:experiments}, we report the results of experiments
with both our best-effort adaptation algorithms and our domain
adaptation algorithms demonstrating their effectiveness, as well as
comparisons with several baselines.  This includes a discussion and
empirical analysis of how our results benefit the design of principled
solutions for \emph{fine-tuning} and other few-shot algorithms
(Section~\ref{app:fine-tuning}).
We start with the introduction of some preliminary definitions and
concepts related to adaptation (Section~\ref{sec:preliminaries}).

\section{Preliminaries}
\label{sec:preliminaries}

We denote by $\sX$ the input space and $\sY$ the output space. In the
regression setting, $\sY$ is assumed to be a measurable subset of $\Rset$. We will
denote by $\sH$ a hypothesis set of functions mapping from $\sX$ to
$\sY$ and by $\ell\colon \sY \times \sY \to \Rset$ a loss function
assumed to take values in $[0, 1]$.

We will study problems with a source domain $\sQ$ and target domain
$\sP$, where $\sQ$ and $\sP$ are distributions over $\sX \times
\sY$. We will denote by $\h \sQ$ the empirical distribution associated
to a sample $S$ of size $m$ drawn from $\sQ^m$ and similarly by $\h \sP$
the empirical distribution associated to a sample $S'$ of size $n$ drawn
from $\sP^n$.  We will denote by $\sQ_X$ and $\sP_X$ the marginal
distributions of $\sQ$ and $\sP$ on $\sX$. We will denote by $\cL(\sP,
h)$ the population loss of a hypothesis over $\sP$ defined as:
$\cL(\sP, h) = \E_{\substack{(x,y) \sim \sP}} \bracket*{\ell((x),
  y)}$.

Several notions of discrepancy have been shown to be adequate
measures between distributions for adaptation problems
\citep{KiferBenDavidGehrke2004,MansourMohriRostamizadeh2009,
  MohriMunozMedina2012,CortesMohri2014, CortesMohriMunozMedina2019}.
We will denote by $\dis(\sP, \sQ)$ the \emph{labeled discrepancy} of
$\sP$ and $\sQ$, also called $\sY$-discrepancy in
\citep{MohriMunozMedina2012,CortesMohriMunozMedina2019} and defined
by:
\begin{equation}
\label{eq:disc-definition}
  \dis(\sP, \sQ)
  =
  \sup_{h \in \sH} 
    \E_{\substack{(x, y) \sim \sP}} \bracket*{\ell(h(x), y)}
    -     \E_{\substack{(x, y) \sim \sQ}} \bracket*{\ell(h(x), y)}.
\end{equation}
Note that we are not using absolute values around the difference of
expectations, as in the original discrepancy definitions in prior work as the one-sided definition suffices for our
analysis. We will denote 
%by a capitalized notation 
the version with
absolute values as: $\Dis(\sP, \sQ) = \max\curl*{\dis(\sP, \sQ),
  \dis(\sQ, \sP)}$. 
  
By definition, computing the labeled discrepancy assumes access to
labels from both $\sP$ and $\sQ$. In contrast, the \emph{unlabeled
discrepancy}, denoted by $\ov \dis(\sP, \sQ)$, requires no access to
such labels 
\begin{align}
  \ov \dis(\sP, \sQ) 
  = \sup_{h, h' \in \sH}
    \E_{\substack{x \sim \sP_X}} \bracket*{\ell(h(x), h'(x))}
    - 
    \E_{\substack{x \sim \sQ_X}} \bracket*{\ell(h(x), h'(x))}.
\end{align}
We will similarly denote by $\ov \Dis(\sP, \sQ)$ the counterpart of
this definition with absolute values.  As shown by
\cite{MansourMohriRostamizadeh2009}, the unlabeled discrepancy can be
accurately estimated from finite (unlabeled) samples from $\sQ_X$ and
$\sP_X$ when $\sH$ admits a favorable Rademacher complexity, for
example a finite VC-dimension. The unlabeled discrepancy is a
divergence measure tailored to (unsupervised) adaptation that can be
upper bounded by the $\ell_1$-distance.  It coincides with the
so-called \emph{$d_A$-distance} introduced by
\cite{KiferBenDavidGehrke2004} in the special case of the zero-one
loss. We will also be using the finer notion of \emph{local labeled
discrepancy} for some suitably chosen subsets $\sH_1$ and $\sH_2$ of
$\sH$:
\begin{align}
  \mspace{-40mu}
  \ov \dis_{\sH_1 \times \sH_2}(\sP, \sQ) 
  =
  \mspace{-10mu}
  \sup_{(h, h') \in \sH_1 \times \sH_2}
    \E_{\substack{x \sim \sP_X}} \bracket*{\ell(h(x), h'(x))}
    -
    \E_{\substack{x \sim \sQ_X}} \bracket*{\ell(h(x), h'(x))}.
    \mspace{-20mu}
\end{align}
Local discrepancy \citep{CortesMohriMunozMedina2019} is defined by a
supremum over smaller sets and is thus a more favorable quantity.
We further extend all the discrepancy definitions just presented to
the case where $\sP$ and $\sQ$ are finite signed measures over $\sX
\times \sY$, using the same expressions as above. 
We also abusively extend the definition of discrepancy to
distributions over sample indices. As an example,
given the samples $S$ and $S'$ and a distribution $\sfq$ over their $[m +
n]$ indices, we define the discrepancy $\dis(\h \sP, \sfq)$ as follows:
$\dis(\h \sP, \sfq)
= \sup_{h \in \sH} \frac{1}{n} 
\sum_{i = m + 1}^{n}
\ell(h(x_i), y_i) - \sum_{i = 1}^{m + n} \sfq_i
\ell(h(x_i), y_i)$.

\section{Discrepancy-based generalization bounds}
\label{sec:disc-theory}

There are many algorithms in adaptation based on various methods for
reweighting sample losses and it is natural to seek a similar solution
for best-effort adaptation (see Appendix~\ref{app:related}). We
present a general theoretical analysis covering all such sample
reweighting methods. We give new discrepancy-based generalization
bounds, including learning bounds holding uniformly over the weights.

We assume that the learner has access to a labeled sample $S =
\paren*{(x_1, y_1), \ldots, (x_m, y_m)}$ drawn from $\sQ^m$ and a
labeled sample $S' = \paren*{(x_{m + 1}, y_{m + 1}), \ldots, (x_{m +
    n}, y_{m + n})}$ drawn from $\sP^n$. In the problems we consider,
we typically have $m \gg n$, but our analysis applies is general.
For a non-negative vector $\sfq$ in $[0, 1]^{[m + n]}$, we denote by
$\ov \sfq$ the total \emph{weight} on the first $m$ points:
$\ov \sfq = \sum_{i = 1}^m \sfq_i$ and by
$\Rad_{\sfq}(\ell \circ \sH)$ the $\sfq$-weighted Rademacher
complexity:
\begin{equation}
\label{eq:q-weighted-Rad}
\Rad_{\sfq}(\ell \circ \sH) 
= \E_{S, S', \bsigma} \bracket*{\sup_{h \in \sH}
\sum_{i = 1}^{m + n} \sigma_i \sfq_i \ell(h(x_i), y_i)},
\end{equation}
where the Rademacher variables $\sigma_i$ are independent random
variables distributed uniformly over $\set{-1,+1}$.  The
$\sfq$-weighted Rademacher complexity is a natural extension of the
Rademacher complexity taking into consideration distinct weights
assigned to sample points. It can be upper-bounded as follows in terms
of the (unweighted) Rademacher complexity: $\Rad_{\sfq}(\ell \circ
\sH) \leq \| \sfq \|_\infty (m + n) \, \Rad_{m + n}(\ell \circ \sH)$,
with equality for uniform weights (see Lemma~\ref{lemma:rad},
Appendix~\ref{app:qgen}).

The following is a general learning guarantee expressed in terms of
the weights $\sfq$.  Note that we do not require $\sfq$ to be a
distribution over $[m + n]$, that is $\norm{\sfq}_1$ may not equal
one.

\begin{restatable}{theorem}{qGeneralTheorem}
\label{th:qgen}
Fix a vector $\sfq$ in $[0, 1]^{[m + n]}$.  Then, for any $\delta > 0$,
with probability at least $1 - \delta$ over the choice of a sample $S$
of size $m$ from $\sQ$ and a sample $S'$ of size $n$ from $\sP$, the
following holds for all $h \in \sH$:
\begin{equation*}
\cL(\sP, h)
 \leq \sum_{i = 1}^{m + n} \sfq_i \ell(h(x_i), y_i)
+ \dis \paren[\Big]{\bracket[\big]{(1 - \norm{\sfq}_1) + \ov \sfq} \sP, \ov \sfq \sQ}
 +  2 \Rad_{\sfq}(\ell \circ \sH)
+ \| \sfq \|_2 \sqrt{\frac{\log \frac{1}{\delta}}{2}}.
\end{equation*}
\end{restatable}
This bound is tight as a function of the discrepancy term, as shown by
the following theorem, which underscores the importance of that
term. The proofs for both theorems are given in
Appendix~\ref{app:qgen}.

\begin{restatable}{theorem}{LowerBound}
  Fix a distribution $\sfq$ in $\Delta_{m + n}$. Then, for any $\e >
  0$, there exists $h \in \sH$ such that, for any $\delta > 0$, the
  following lower bound holds with probability at least $1 - \delta$
  over the choice of a sample $S$ of size $m$ from $\sQ$ and a sample
  $S'$ of size $n$ from $\sP$:
\begin{equation*}
\cL(\sP, h) 
 \geq \sum_{i = 1}^{m + n} \sfq_i \ell(h(x_i), y_i)
+ \ov \sfq \dis(\sP, \sQ)
 - 2 \Rad_{\sfq}(\ell \circ \sH)
- \| \sfq \|_2 \sqrt{\frac{\log \frac{1}{\delta}}{2}}
-\e.
\end{equation*}
In particular, for $\norm{\sfq}_2,\Rad_{\sfq}(\ell \circ \sH) \in
O\paren{\frac{1}{\sqrt{m + n}}}$, we have:
\[
\cL(\sP, h) \geq \sum_{i = 1}^{m + n} \sfq_i \ell(h(x_i), y_i)
+ \ov \sfq \dis(\sP, \sQ) + \Omega\paren*{\frac{1}{\sqrt{m + n}}}.
\]
\end{restatable}

The bound of Theorem~\ref{th:qgen} cannot be used to choose $\sfq$
since it holds for a fixed choice of that vector. A standard way to
derive a uniform bound over $\sfq$ is via covering numbers.
%In this context, 
That requires applying the union bound to the
centers of an $\e$-covering of $[0, 1]^{[m + n]}$ for the $\ell_1$
distance. But, the corresponding covering number $\cN_1$ would be in
$O(\paren*{1/\e}^{m + n})$, resulting in an
uninformative bound, even for $\norm{\sfq}_2 = 1/\sqrt{m + n}$,
since $\sqrt{\log \cN_1/m + n}$ would be a constant.  Instead,
we present an alternative analysis, generalizing
Theorem~\ref{th:qgen} to hold uniformly over $\sfq$ in $\curl*{\sfq
  \colon 0 < \| \sfq - \sfp^0 \|_1 < 1}$, where $\sfp^0$ could be
interpreted as a reference (or ideal) reweighting choice.
The proof is presented in Appendix~\ref{app:qgen}.

\begin{restatable}{theorem}{qGeneralTheoremUniform}
\label{th:qgen-uniform}
For any $\delta > 0$, with probability at least $1 - \delta$ over the
choice of a sample $S$ of size $m$ from $\sQ$ and a sample $S'$ of
size $n$ from $\sP$, the following holds for all $h \in \sH$ and $\sfq
\in \curl*{\sfq \colon 0 \leq \| \sfq - \sfp^0 \|_1 < 1}$:
\begin{align*}
& \cL(\sP, h) 
 \leq \sum_{i = 1}^{m + n} \sfq_i \ell(h(x_i), y_i)
+ \dis\paren[\Big]{\bracket[\big]{(1 - \norm{\sfq}_1) + \ov \sfq} \sP,
  \ov \sfq \sQ} + \dis(\sfq, \sfp^0)\\
& + 2 \Rad_{\sfq}(\ell \circ \sH)
+ 5 \norm{\sfq - \sfp^0}_1 
 + \bracket*{\norm{\sfq}_2 + 2 \norm{\sfq - \sfp^0}_1}
\bracket*{ \sqrt{\log \log_2 \tfrac{2}{1 - \norm{\sfq - \sfp^0}_1}}
  + \sqrt{\tfrac{\log \frac{2}{\delta}}{2}} }.
\end{align*}
\end{restatable}
\noindent Note that for $\sfq = \sfp^0$, the bound coincides with that of
Theorem~\ref{th:qgen}.\\

\noindent\textbf{Learning bounds insights}. Theorems~\ref{th:qgen} and
\ref{th:qgen-uniform} provide general guarantees for best-effort
adaptation. They suggest that for adaptation to succeed via sample
reweighting, a favorable balance of \emph{several key terms} is
important.
The first term suggests minimizing the $\sfq$-weighted empirical
loss. However, the bound advises against doing so at the price of
assigning non-zero weights only to a small fraction of the points
since that would increase the $\norm{\sfq}_2$ term. In fact, a
comparison with the familiar inverse of square-root of the sample size
term appearing in other bounds suggests interpreting
$\paren*{1/\norm{\sfq}^2_2}$ as the \emph{effective sample size}.
Note also that when $\sfq$ is a distribution, the second term admits
the following simpler form: $\dis\paren*{\bracket*{(1 - \norm{\sfq}_1)
    + \ov \sfq} \sP, \ov \sfq \sQ} = \dis(\ov \sfq \sP, \ov \sfq \sQ)
= \ov \sfq \dis(\sP, \sQ)$. Thus, the second term of these bounds
suggests allocating less weight to the points drawn from $\sQ$, when
the discrepancy $\dis(\sP, \sQ)$ is large.
The weighted discrepancy term $\dis(\sfq, \sfp^0)$ and the
$\ell_1$-distance $\norm{\sfq - \sfp_0}_1$ in
Theorem~\ref{th:qgen-uniform} both press $\sfq$ to be chosen
relatively closer to the reference $\sfp^0$.  Finally, the Rademacher
complexity term is a familiar measure of the complexity of the
hypothesis set, which here additionally takes into consideration the
weights.

In Appendix~\ref{app:discussion}, we compare the bound of
Theorem~\ref{th:qgen} with some existing discrepany-based ones and
show how they can be recovered as special cases. In particular, we
show that the discrepancy-based bound of
\cite{CortesMohriMunozMedina2019}, which is the basis for the
discrepancy minimization algorithm of \cite{CortesMohri2014}, is
always an upper bound on a special case (specific choice of the
weights) of the bound of Theorem~\ref{th:qgen}.

We note that assigning non-uniform weights to the points in $S$ should
not be viewed as unnatural, even though the points are sampled from the
same distribution. This is because these weights serve to
make the $\sfq$-weighted empirical loss closer to the empirical loss
for the target sample. As an example, importance weighting seeks
distinct weights for each point based on the source and target
distributions.
Nevertheless, in Appendix~\ref{app:discussion}, we consider a simple
\emph{$\alpha$-reweighting} method, which allocates uniform weights to
source  points. We show that, under some assumptions, even for
this very simple choice of the weights, the learning bound can be more
favorable than the one for training only on target samples.

Theorem~\ref{th:qgen-uniform} suggests choosing $h \in \sH$ and $\sfq
\in \curl{\sfq \colon 0 \leq \| \sfq - \sfp^0 \|_1 < 1}$ to minimize
the right-hand side of the inequality and seek the best balance
between these key terms. This guides the design of our learning
algorithms. The following corollary provides a slightly
simplified version of Theorem~\ref{th:qgen-uniform} (see
Appendix~\ref{app:qgen}).

\begin{restatable}{corollary}{qGeneralCorollary}
\label{cor:qgen-uniform}
For any $\delta > 0$, with probability at least $1 - \delta$ over the
choice of a sample $S$ of size $m$ from $\sQ$ and a sample $S'$ of
size $n$ from $\sP$, the following holds for all $h \in \sH$ and $\sfq
\in \curl*{\sfq \colon 0 \leq \| \sfq - \sfp^0 \|_1 < 1}$:
\begin{align*}
\cL(\sP, h)
& \leq \sum_{i = 1}^{m + n} \sfq_i \ell(h(x_i), y_i)
+ \ov \sfq \dis(\sP, \sQ)
+ \dis(\sfq, \sfp^0)
+ 2 \Rad_{\sfq}(\ell \circ \sH)\\
& \quad + 6 \norm{\sfq - \sfp^0}_1 
+ \bracket*{\norm{\sfq}_2 + 2 \norm{\sfq - \sfp^0}_1}
\bracket*{ \sqrt{\log \log_2 \tfrac{2}{1 - \norm{\sfq - \sfp^0}_1}}
  + \sqrt{\frac{\log \tfrac{2}{\delta}}{2}} }.
\end{align*}
\end{restatable}

\section{Best-Effort adaptation algorithms}
\label{sec:disc-algorithms}

In this section, we describe new learning algorithms for best-effort
adaptation directly benefiting from the theoretical analysis of the
previous section.\\

\noindent\textbf{Optimization problem, \best\ and \sbest\ algorithms}.
The previous section suggests seeking $h \in \sH$ and $\sfq \in [0,
  1]^{m + n}$ to minimize the bound of Theorem~\ref{th:qgen-uniform}
or that of Corollary~\ref{cor:qgen-uniform}. To simplify the discussion,
we will focus on the algorithm derived from Corollary~\ref{cor:qgen-uniform}.
A similar but finer algorithm consists instead of using directly 
Theorem~\ref{th:qgen-uniform}.

Assume that $\sH$ is a subset of a normed vector space and that the
Rademacher complexity term can be bounded by an upper bound
on the norm squared $\norm{h}^2$.  Then, using the shorthand
$d_i = \dis(\sP, \sQ) 1_{i \in [m]}$, the optimization problem can
be written as:
\begin{align*}
  \min_{h \in \sH, \sfq \in [0, 1]^{m + n}}
& \sum_{i = 1}^{m + n} \sfq_i
  \bracket*{\ell(h(x_i), y_i) + d_i}
+ \dis(\sfq, \sfp^0) 
+  \lambda_\infty \norm{\sfq}_\infty \norm{h}^2\\
& + \lambda_1 \norm{\sfq - \sfp^0}_1 + \lambda_2 \norm{\sfq}^2_2,
\end{align*}
where $\lambda_1$, $\lambda_2$ and $\lambda_\infty$ are non-negative
hyperparameters. A natural choice for $\sfp^0$ in 
our scenario is the uniform distribution over $S'$, which is the 
empirical distribution in the absence of any point from a different
distribution $\sQ$.
We will refer by \best\ to an algorithm seeking to minimize this
objective.  We will also consider a simpler version of our algorithm,
\sbest, where we upper-bound $\dis(\sfq, \sfp^0)$ by $\norm{\sfq -
  \sfp^0}_1$, in which case this additional term is subsumed by the
existing one with $\lambda_1$ factor.

When the loss function $\ell$ is convex with respect to its first
argument, the objective function is convex in $h$ and in $\sfq$. In
particular, $\dis(\sfq, \sfp^0)$ is a convex function of $\sfq$ as a
supremum of convex functions (affine functions in $\sfq$): $\dis(\sfq,
\sfp^0) = \sup_{h \in \sH} \curl*{\sum_{i = 1}^{m + n} (\sfq_i -
  \sfp^0_i) \ell(h(x_i), y_i)}$.
But, the objective function is not jointly convex.\\

\noindent\textbf{Alternating minimization solution}. One method for solving the problem
consists of alternating minimization (or block coordinate descent),
that is of minimizing the objective over $\sH$ for a fixed value of
$\sfq$ and next of minimizing with respect to $\sfq$ for a fixed value
of $h$. In general, this method does not benefit from convergence
guarantees, although there is a growing body of literature proving
guarantees under various assumptions
\citep{GrippoSciandrone2000,LiZhuTang2019,Beck2015}.  \\

\noindent\textbf{DC-programming solution}. An alternative solution consists of
casting the problem as an instance of DC-programming (difference of
convex) by rewriting the objective as a difference.
Note that for any non-negative and convex function $f$ and
any non-decreasing and convex function $\Psi$ defined over
$\Rset_+$, $\Psi \circ f$ is convex:
for all $(x, x') \in \sX^2$ and $\alpha \in [0, 1]$,
\begin{align*}
 (\Psi \circ f)(\alpha x + (1 - \alpha) x')
& \leq \Psi(\alpha f(x) + (1 - \alpha) f(x'))\\
%\tag{\text{convexity of $f$ and $\Psi$ non-decreasing}}
& \leq \alpha (\Psi \circ f) (x) + (1 - \alpha) (\Psi \circ f) (x'),
%\tag{\text{convexity of $\Psi$}}
\end{align*}
where the first inequality holds by the convexity of $f$ and the
non-decreasing property of $\Psi$ and the last one by the convexity
of $\Psi$.
In particular, for any non-negative and convex function $f$, $f^2$ is
convex. Thus, we can rewrite the non-jointly convex terms of the
objective as the following DC-decompositions:
\begin{align*}
\sfq_i \ell(h(x_i), y_i)
& = \frac{1}{2} \bracket*{
\bracket*{\sfq_i + u}^2 - \bracket*{\sfq_i^2 + u^2}},\\
\norm{\sfq}_\infty \norm{h}^2
& = \frac{1}{2} \bracket*{
\bracket*{\norm{\sfq}_\infty + \norm{h}^2}^2
- \bracket*{\norm{\sfq}^2_\infty + \norm{h}^2}
},
\end{align*}
where $u = \ell(h(x_i), y_i)$.  We can then use the DCA algorithm of
\cite{TaoAn1998}, (see also \cite{TaoAn1997}), which in our
differentiable case coincides with the CCCP algorithm of
\cite{YuilleRangarajan2003}, further analyzed by
\cite{SriperumbudurTorresLanckriet2007}. The DCA algorithm guarantees
convergence to a critical point.  The global optimum can be found by
combining DCA with a branch-and-bound or cutting plane method
\citep{Tuy1964,HorstThoai1999,TaoAn1997}.
\ignore{
We also present a solution based on convex optimization in the case of
the squared loss with a linear or kernel-based hypothesis set
(Appendix~\ref{app:convex}).\\
}

\noindent\textbf{Discrepancy estimation}. Our algorithm requires estimating the
discrepancy terms. We discuss our DC-programming solution to this
problem in detail in Appendix~\ref{app:disc-est}.

As already pointed, our learning bounds are general and can be used
for the analysis of various specific reweighting methods with bounded
weights, including discrepancy minimization \citep{CortesMohri2014},
KMM \citep{HuangSmolaGrettonBorgwardtScholkopf2006}, KLIEP
\citep{SugiyamaEtAl2007a}, importance weighting
\citep{CortesMansourMohri2010}, when the weights are bounded, and many
others. However, unlike our algorithms, which simultaneously learn the
weights and the hypothesis and directly benefit from the learning
bounds of the previous section, these algorithms typically consist of
two stages and do not exploit the guarantees discussed: in the first
stage, they determine some weights $\sfq$, irrespective of the labeled
samples and the empirical loss; in the second stage, they use these
weights to learn a hypothesis minimizing the $\sfq$-weighted empirical
loss. Additionally, some methods admit other specific drawbacks. For example, it was shown by \cite{CortesMansourMohri2010}, both theoretically and empirically, that, in general, importance weighting may not succeed. Note also that the method relies only on the ratio of the densities and does not take into account, unlike the discrepancy, the hypothesis set and the loss function.

\section{Domain adaptation}
\label{sec:da}

The analysis of Section~\ref{sec:disc-theory} can also be used to
derive general discrepancy-based guarantees for domain adaptation,
where the learner has access to few or no labeled points from the
target domain. In this section, we analyze the case where the input
points in $S'$ are unlabeled. Our analysis can be straightforwardly
extended to the case where a small fraction of the labels in $S'$ are
available. Our theoretical analysis leads to the design of new
algorithms for domain adaptation.

\subsection{Domain adaptation generalization bounds}
\label{sec:da-bounds}

For convenience, in this section, we will use a different notation for
the weights on $S$ and $S'$: $\sfq \in [0, 1]^{m}$ for the weights on
$S$, $\sfq' \in [0, 1]^{n}$ for the weights on $S'$.
The labels of the points in $S'$ appear in the first term of
the bound of Theorem~\ref{th:qgen}, the $\sfq$-weighted empirical
loss. Since they are not available, we upper-bound
the empirical loss in terms of a $\sfp$-weighted
empirical loss and a discrepancy term:
\begin{equation}
\label{eq:disc-ineq}
\sum_{i = 1}^{m} \sfq_i \ell(h(x_i), y_i)
+ \sum_{i = 1}^{n} \sfq'_i \ell(h(x_{m + i}), y_{m + i}) 
\leq \sum_{i = 1}^{m} (\sfq_i + \sfp_i) \ell(h(x_i), y_i)
+ \dis(\sfq', \sfp),
\end{equation}
for any weight vector $\sfp \in [0, 1]^m$.
This yields
immediately the following theorem.

\begin{restatable}{theorem}{qDA}
\label{th:da}
Fix the vectors $\sfq$ in $[0, 1]^{[m]}$ and $\sfq' \in [0, 1]^{n}$.
Then, for any $\delta > 0$, with probability at least $1 - \delta$
over the choice of a sample $S$ of size $m$ from $\sQ$ and a sample
$S'$ of size $n$ from $\sP$, the following holds for all $\sfp$ in
$[0, 1]^{[m]}$ and $h \in \sH$:
\begin{align*}
\cL(\sP, h)
& \leq \sum_{i = 1}^{m} (\sfq_i + \sfp_i) \ell(h(x_i), y_i)
+ \dis(\sfq', \sfp)
+ \dis \paren[\Big]{\bracket[\big]{1 - \norm{\sfq'}_1} \sP, \norm{\sfq}_1 \sQ}\\
& \quad +  2 \Rad_{(\sfq, \sfq')}(\ell \circ \sH)
+  \sqrt{\frac{\paren*{\norm{\sfq}^2_2 + \norm{\sfq'}^2_2}
    \log \frac{1}{\delta}}{2}}.
\end{align*}
\end{restatable}
This learning bound can be extended to hold uniformly over 
\[
\curl*{(\sfq, \sfq') \in [0, 1]^m \times [0, 1]^n \colon 0 < \|
  (\sfq, \sfq') - \sfp^0 \|_1 < 1}
\]
and all $\sfp$ in $[0, 1]^{[m]}$,
where $\sfp^0$ is a reference (or ideal) reweighting choice over the
$(m + n)$ points (see Theorem~\ref{th:da-unif} and
Corollary~\ref{cor:da-unif-simplified} in Appendix~\ref{app:da-unif}).
Note that, here, both $\sfp$ and $\sfq'$ can be chosen to make the
weighted-discrepancy term $\dis(\sfq', \sfp)$ smaller.  Several of the
comments on Theorem~\ref{th:qgen} similarly apply here. In particular,
it is worth pointing out that the learning bound of 
\cite{CortesMohriMunozMedina2019} can be recovered for a specific
choice of the weights.
This holds even in the special
case where $\sfq = 0$ and where $\sfq'$ is a distribution:
\begin{align*}
\cL(\sP, h)
 \leq \sum_{i = 1}^{m} \sfp_i \ell(h(x_i), y_i)
+ \dis(\sfq', \sfp)+  2 \Rad_{\sfq'}(\ell \circ \sH)
+  \norm{\sfq'}_2 \sqrt{\frac{\log \frac{1}{\delta}}{2}}.
\end{align*}
In that case, choosing $\sfq'$ to be the empirical distribution on
$S'$ leads to the bound of \cite{CortesMohriMunozMedina2019} (see also inequality \eqref{eq:P-bound2}, in Appendix~\ref{app:discussion}). An alternative choice of the
weights may lead to a smaller discrepancy term $\dis(\sfq', \sfp)$ and
a better guarantee overall. Our learning algorithm will seek an optimal
choice for the weights.

The discrepancy quantities appearing in the bound of the
theorem cannot be estimated in the absence of labels for
$S'$.  Thus, we need to resort to upper-bounds expressed in terms of
unlabeled discrepancies, using only unlabeled data from $\sP$.
A detailed analysis is presented in Appendix~\ref{app:disc-upper}.

\subsection{Domain adaptation \bda\ algorithm}
\label{sec:da-algorithm}

The analysis of the previous section suggests seeking $h \in \sH$,
$\sfq$ and $\sfp$ in $[0, 1]^m$ and $\sfq'$ in $[0, 1]^n$ to minimize
the bound of Theorem~\ref{th:da-unif} or that of
Corollary~\ref{cor:da-unif-simplified}.  As in
Section~\ref{sec:disc-algorithms}, assume that $\sH$ is a subset of a
normed vector space and that the Rademacher complexity term can be
bounded in terms of an upper bound on the norm squared $\norm{h}^2$.
Then, the optimization problem corresponding to
Corollary~\ref{cor:da-unif-simplified} can be written as follows:
\begin{align}
\label{eq:daopt}
  \min_{\substack{h \in \sH, \sfq, \sfp \in [0, 1]^m\\ \sfq' \in [0, 1]^n}}
& \sum_{i = 1}^{m} (\sfq_i + \sfp_i) \, \ell(h(x_i), y_i) 
 + \norm{\sfq}_1 \ov d
  + \ov \dis(\sfq', \sfp)
  + \ov \dis((\sfq, \sfq'), \sfp^0) \\\nonumber
& +  \lambda_\infty \norm{(\sfq, \sfq')}_\infty \, \norm{h}^2
  + \lambda_1 \norm{(\sfq, \sfq')  - \sfp^0}_1
 + \lambda_2 (\norm{\sfq}^2_2 + \norm{\sfq'}^2_2),
\end{align}
where $\lambda_1$, $\lambda_2$ and $\lambda_\infty$ are non-negative
hyperparameters and where we used the shorthand $\ov d = \ov \dis(\sP,
\sQ)$. We are omitting subscripts to simplify the presentation but,
as discussed in the previous section, the unlabeled discrepancies
in the optimization problem may be local unlabeled discrepancies, which
are finer quantities.
As in the best-effort adaptation, a natural choice for $\sfp^0$
in the domain adaptation scenario is the uniform distribution
over the input points of $S'$. In practice, specific applications may
motivate better choices.

We will refer by \bda\ to the algorithm seeking to minimize this
objective.  
% We will also consider a simpler version of our algorithm,
% \sbda, where we upper bound $\ov \dis((\sfq, \sfq'), \sfp^0)$ by
% $\norm{(\sfq, \sfq') - \sfp^0}_1$, in which case this additional term is
% subsumed by the existing one with $\lambda_1$ factor.
Our comments and analysis of the \best\ optimization
(Section~\ref{sec:disc-algorithms}) apply similarly here. In
particular, the problem can be similarly cast as a DC-programming
problem or a convex optimization problem. The unlabeled discrepancy
term $\ov d = \ov \dis(\sP, \sQ)$ can be accurately estimated from
$\ov \dis(\sP, \sQ)$. In Appendix~\ref{app:udisc}, we show in
detail how to compute $\ov \dis(\sP, \sQ)$ and how to
evaluate the sub-gradients of the weighted discrepancy terms.
% $\ov \dis(\sfq', \sfp)$ and
% $\ov \dis((\sfq, \sfq'), \sfp^0)$.

\subsection*{Discussion of new \bda\ algorithm}

Our \bda\ algorithm benefits from more favorable guarantees than
previous discrepancy-based algorithms
\citep{MansourMohriRostamizadeh2009,CortesMohri2014,
  CortesMohriMunozMedina2019} and algorithms seeking to minimize the
learning bound \eqref{eq:P-bound2}, with the unlabeled discrepancy
upper bounded by the label discrepancy.  This is because, as already
pointed out, \bda\ is based on a learning guarantee that admits as a
special case \eqref{eq:P-bound2}.  Thus, the best choice of the
weights and predictor sought by the algorithm include those
corresponding to previous algorithms as a special case.

Moreover, as discussed in Section~\ref{sec:disc-theory}, our upper
bounds in terms of local discrepancy are finer than those used in
previous work.  In particular, \bda\ improves upon the \dm\ algorithm
(\emph{discrepancy minimization}) of \cite{CortesMohri2014}, which
has been shown empirically by the authors to outperform other domain
adaptation baselines in regression tasks. \dm\ seeks to minimize
\eqref{eq:P-bound2} via a two-stage method, by first seeking weights
that minimize the unlabeled weighted-discrepancy (second term) and
subsequently seeking $h \in \sH$ to minimize the empirical loss for
that fixed choice of $\sfq$. This two-stage method may be suboptimal,
compared to an algorithm seeking to directly minimize the bound to
find $(h, \sfq)$.  The solution $\sfq$ found to minimize the
discrepancy term in the first stage may, for example, assign
significantly larger weights to some sample points, which could lead
to a poor choice of the predictor in the second stage.

An alternative sophisticated technique based on the so-called
\emph{generalized discrepancy} is advocated by
\cite{CortesMohriMunozMedina2019}.  The main benefit of this
technique is to allow for the weights to be chosen as a function of
the hypotheses, unlike the two-stage \dm\ solution of
\cite{CortesMohri2014}. Our \bda\ algorithm, however, already offers
that advantage since the hypothesis $h$ and the weights $\sfq$,
$\sfq'$ and $\sfp$ are sought simultaneously as a solution of the
optimization problem. Note, however that the choice of the weights in
the generalized discrepancy method does not take into consideration
the empirical losses, unlike our algorithm.  Furthermore,
\bda\ minimizes a learning bound admitting as a special case
\eqref{eq:P-bound2}, the best learning guarantee presented by the
authors in support of their algorithm. Let us add that authors state
that their guarantee for the generalized discrepancy method is not
comparable to that of \dm\ algorithm.

\subsection{Labeled discrepancy upper bounds}
\label{app:disc-upper}

The analysis of Section~\ref{sec:disc-theory} is based on the labeled
discrepancy measure $\dis(\sP, \sQ)$ or its estimate from finite
samples $\dis(\h \sP, \h \sQ)$, which assumes access to labeled data
from the target distribution $\sP$. In typical domain adaptation
problems, however, there is little labeled data or none from the
target domain $\sP$.  Thus, instead we need to resort to an
upper-bound on $\dis(\sP, \sQ)$ in terms of the unlabeled discrepancy,
which only uses unlabeled data from $\sP$.

We will discuss two types of upper bounds, first in the special case
of the squared loss, next in the case of an arbitrary $\mu$-Lipschitz
loss. Our analysis benefits from that of previous work
\citep{CortesMohri2014,CortesMohriMunozMedina2019} but
improves upon that, as discussed later.

\noindent\textbf{Squared loss.}
Here, we give an upper bound on the labeled discrepancy in the
case of the squared loss.
For any hypothesis $h_0 \in \sH$, we denote by $\delta_{\sH, h_0}(\h \sP,
\h \sQ)$ the \emph{squared-loss label discrepancy} of $\h \sP$ and $\h
\sQ$:
\begin{align}
\label{eq:delta-label-discrepancy}
\delta_{\sH, h_0}(\h \sP, \h \sQ) 
= 
\sup_{h \in \sH} \, \abs{ \E_{\substack{(x, y) \sim \h \sP}} \bracket*{h(x) \paren*{y - h_0(x)}}
 - \E_{\substack{(x, y) \sim \h \sQ}} \bracket*{h(x) \paren*{y - h_0(x)}} }.
\end{align}

\begin{restatable}{lemma}{DiscUpperBoundOne}
\label{lemma:disc-upper-bound1}
  Let $\ell$ be the squared loss. Then, for any hypothesis $h_0$ in
  $\sH$, the following upper bound holds for the labeled discrepancy:
  \[
\dis(\h \sP, \h \sQ) 
\leq \ov \dis_{\sH \times \set{h_0}}(\h \sP, \h \sQ)
+ 2 \delta_{\sH, h_0}(\h \sP, \h \sQ).
  \]
\end{restatable}
\noindent The proof is given below in Appendix~\ref{app:disc-upper-bound1}.
The local unlabeled discrepancy $\ov \dis_{\sH \times \set{h_0}}(\h
\sP, \h \sQ)$ captures the closeness of the input distributions $\h
\sP_X$ and $\h \sQ_X$. It is a significantly more favorable term that
the standard unlabeled discrepancy since it admits only a maximum over
$h \in \sH$ and not over both $h$ and $h'$ in $\sH$.

For a suitable choice of $h_0 \in \sH$, the term $\delta_{\sH, h_0}(\h \sP, \h
\sQ)$ captures the closeness of the empirical output labels on $\h \sP$
and $\h \sQ$. Note that for $\h \sP = \h \sQ$,
we have $\delta_{\sH, h_0}(\h \sP, \h \sQ) = 0$ for any $h_0 \in \sH$. 
When the covariate-shift assumption holds and the problem is separable,
$h_0$ can be chosen so that $\delta_{\sH, h_0}(\h \sP, \h \sQ) = 0$. More
generally, when $h_0$ can be chosen so that $|y - h_0(x)|$ is relatively
small on both samples corresponding to $\h \sP$ and $\h \sQ$ and the
hypotheses $h \in \sH$ are bounded by some $M > 0$, then
$\delta_{\sH, h_0}(\h \sP, \h \sQ)$ is relatively small. Note that adaptation
is in general not possible if the learner receives vastly different
labels on the source domain $\sQ$ than those corresponding to the target
$\sP$.

\noindent\textbf{$\mu$-Lipschitz loss.}
Here, we give an upper bound on the labeled discrepancy for any
$\mu$-Lipschitz loss.
For any hypothesis $h_0 \in \sH$, we denote by $\eta_{\sH, h_0}(\h \sP, \h
\sQ)$ the \emph{Lipschitz loss labeled discrepancy} defined by
\begin{align}
\label{eq:eta-label-discrepancy}
\eta_{\sH, h_0}(\h \sP, \h \sQ)
= \E_{(x, y) \sim \h \sP}\bracket*{\abs*{y - h_o(x)}}
+ \E_{(x, y) \sim \h \sQ}\bracket*{\abs*{y - h_o(x)}}.
\end{align}

\begin{restatable}{lemma}{DiscUpperBoundTwo}
\label{lemma:disc-upper-bound2}
  Let $\ell$ be a loss function that is $\mu$-Lipschitz with respect
  to its second argument. Then, for any hypothesis $h_0$ in $\sH$, the
  following upper bound holds for the labeled discrepancy:
  \[
\dis(\h \sP, \h \sQ)
\leq \ov \dis_{\sH \times \set{h_0}}(\h \sP, \h \sQ)
+ \mu \, \eta_{\sH, h_0}(\h \sP, \h \sQ).
  \]
\end{restatable}
The proof is given below in Appendix~\ref{app:disc-upper-bound2}.\\

The Lipschitz loss labeled discrepancy $\eta_{\sH, h_0}(\h \sP, \h
\sQ)$ is a coarser quantity than $\delta_{\sH, h_0}(\h \sP, \h \sQ)$.
In particular, even when $\h \sP = \h \sQ$, $\eta_{\sH, h_0}(\h \sP,
\h \sQ)$ is not zero. However, as with $\delta_{\sH, h_0}(\h \sP, \h
\sQ)$ it captures the closeness of the output labels on $\h \sP$ and
$\h \sQ$.  When $h_0$ can be chosen so that the sum of expected values
$|y - h_0(x)|$ is relatively small on both samples corresponding to
$\h \sP$ and $\h \sQ$ then, $\eta_{\sH, h_0}(\h \sP, \h \sQ)$ is
relatively small. As already pointed out, adaptation is not possible
when the learner received very different labels on the two domains.

The upper bounds of Lemmas~\ref{lemma:disc-upper-bound1} and
\ref{lemma:disc-upper-bound2} hold in the stochastic setting and are
thus more general than those derived for the deterministic label
setting in previous work
\citep{CortesMohri2014,CortesMohriMunozMedina2019}. They are also
finer bounds expressed in terms of the more favorable local
discrepancy and somewhat more favorable label discrepancy terms
defined in terms of expectation over the empirical distributions as
opposed to a supremum.

In both the squared loss and Lipschitz cases, when a relatively small
labeled sample $S'$ drawn i.i.d.\ from $\sP$ is available, we can use
it to select $h_0$ via
\[
h_0 = \argmin_{h_0 \in \sH} \delta_{\sH,
  h_0}(\h \sP_{S'}, \h \sQ) \text{ or }
h_0 = \argmin_{h_0 \in \sH}
\eta_{\sH, h_0}(\h \sP_{S'}, \h \sQ).
\]
When no labeled data from the target domain is at our disposal, we
cannot choose $h_0$ by leveraging any existing information.  We can
then assume that $\min_{h_0 \in \sH} \delta_{\sH, h_0}(\h \sP, \h \sQ)
\ll 1$ in the squared loss case or $\min_{h_0 \in \sH} \eta_{\sH,
  h_0}(\h \sP, \h \sQ) \ll 1$ in the Lipschitz case, that is that the
source labels are relatively close to the target ones based on these
measures and use the standard unlabeled discrepancy:
\begin{align*}
\dis(\h \sP, \h \sQ) 
& \leq \ov \dis(\h \sP, \h \sQ)
+ 2 \min_{h_0 \in \sH} \delta_{\sH, h_0}(\h \sP, \h \sQ)\\
\dis(\h \sP, \h \sQ) 
& \leq \ov \dis(\h \sP, \h \sQ)
+ \mu \min_{h_0 \in \sH} \eta_{\sH, h_0}(\h \sP, \h \sQ).
\end{align*}

\section{Experimental evaluation}
\label{sec:experiments}

We evaluated our algorithms in best-effort adaptation, fine-tuning,
and (unsupervised) domain adaptation. We performed cross-validation
using labeled data from the target to pick the hyperparameters for our
algorithms and the baselines. See Appendix \ref{app:exp-details} for
details on data and experimental procedures. For all the experiments we use the \sbest\ algorithm.

\subsection{Best-Effort adaptation}
\label{sec:beff-exp}

Here we have labeled data both from the source and the target. 
%domains.
Two natural baselines are to train solely on $\sP$, or
 solely $\sQ$. A third baseline is the
$\alpha$-reweighted $\sfq$ as described in
Appendix~\ref{sec:alpha-reweighted}. 
% Note that $\alpha = 1$ corresponds to
% training on all the available data with a uniform weighting.

\setlength{\intextsep}{-12pt}
\setlength{\columnsep}{6pt}
\begin{wrapfigure}{r}{0.3\textwidth}
\begin{center}
    \includegraphics[width=\linewidth]{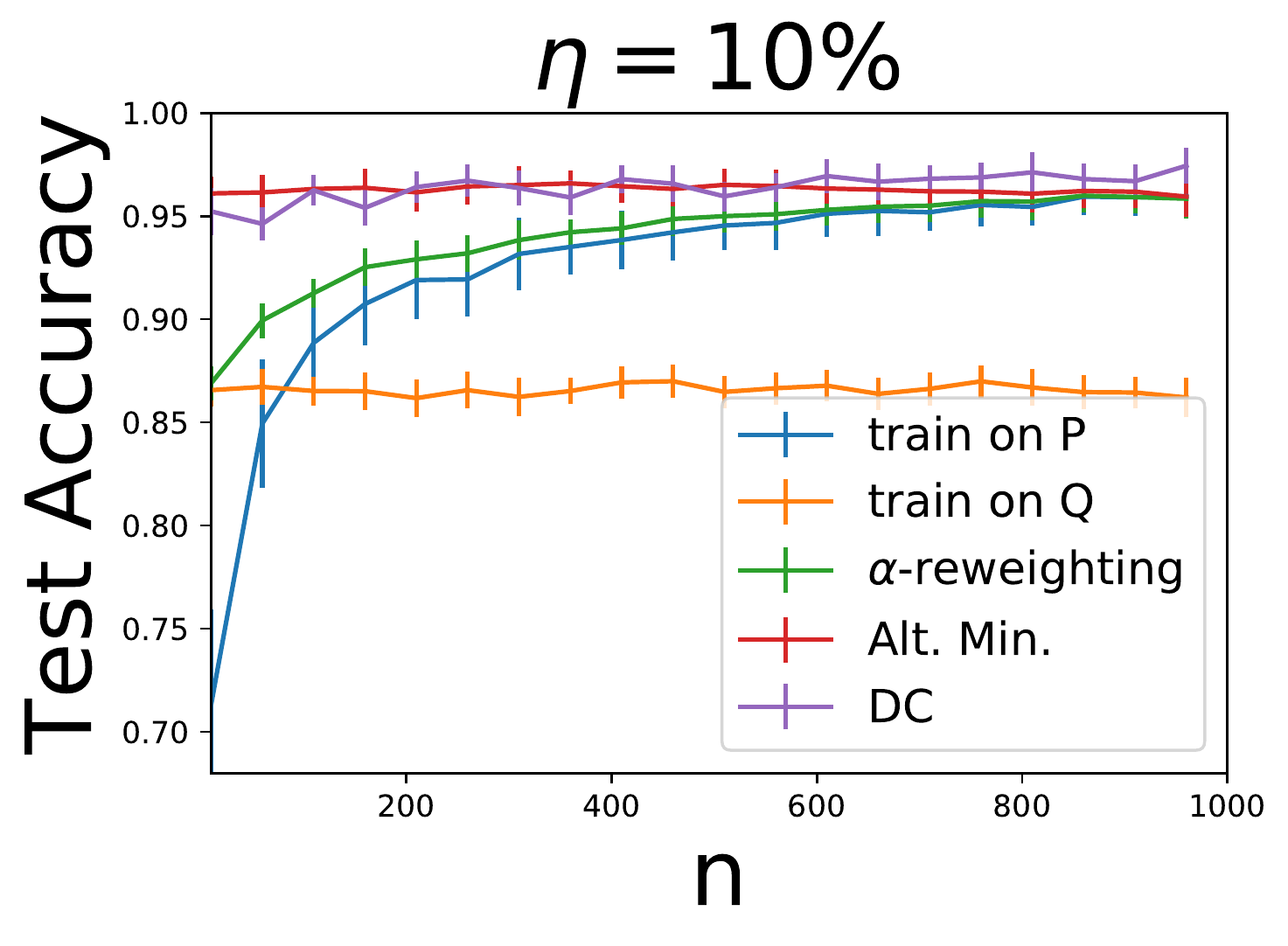}
  \end{center}
  \vskip -.05in
  \caption{Simulated data.}
  \vskip .2in
  \label{fig:eta10}
\end{wrapfigure}
\noindent \textbf{Simulated data.}  The goal of this experiment was to
demonstrate that \sbest\ outperforms the simple baselines just
mentioned and to compare the performance of the Alternate Minimization
(\sbest-{AM}) and the DC-programming (\sbest-DC) optimization
solutions.

We consider a linear
binary classification task with the labels for $\sP$ generated as
$\text{sgn}(w_p\cdot x)$ for a randomly chosen unit vector $w_p$. The
distribution $\sQ$ admits two parts. For $\eta \in (0.5, 1)$, $(1 - \eta)m$
examples are labeled according to $\text{sgn}(w_q \cdot x)$ where
$\|w_p - w_q\| \leq \epsilon$, while the remaining examples are set to
a fixed vector $u$ and labeled $+1$. These $\eta m$ examples represent
the noise in $\sQ$ and, as $\eta$ increases, $\dis(\sP, \sQ)$ gets
larger. For this setting, we evaluated the baselines and  $\sbest$  with the logistic loss
and linear hypotheses. See Appendix~\ref{app:exp-details} for more  details and examples.

Figure~\ref{fig:eta10} shows the performance for $\eta = 10\%$ as $n$
increases. For small sizes, $n$, of the target data $\sP$, both $\alpha$-reweighting and the
baseline that trains solely on $\sQ$ are significantly impacted. This
is because these methods cannot distinguish between non-noisy and
noisy data points.  On the other hand, both \sbest-AM and \sbest-DC
can counter the effect of the noise by generating $\sfq$-weights that
are predominantly supported on the non-noisy samples. The performance
of these algorithms is fairly independent of the size of $n$ as, for
$\eta = 10\%$, they can still make an effective use of 90\% of the $m
= 1000$ examples.  As $n$ increases, $\alpha$-reweighting and the
baseline that trains solely on $\sP$ reach the performance of
\sbest. We also note that \sbest-AM and \sbest-DC perform
equivalently and in all the following experiments, we use \sbest-AM. For experiments with other values of $\eta$ and further
discussion of this experiment, see Appendix~\ref{app:exp-details}.\\

\subsection{Fine-tuning tasks}
\label{sec:finetuning-exp}

\begin{table}[t]
\caption{Performance of \sbest, compared to baseline
  approaches in CIFAR-10.}
\begin{center}
%\vskip -.1in
\begin{tabular}{@{\hspace{0cm}}lllll@{\hspace{0cm}}}
%\hrule
Fine-tuning & \thead{Train on $\sP$}&  \thead{gapBoost} & \thead{\sbest}\\
%\hrule
Last layer (CIFAR-10) & $88.61 \pm .43$  & $87.1 \pm .01$ & $\mathbf{89.62 \pm .32}$\\
Full model (CIFAR-10) & $90.18 \pm .31$&  $90.8 \pm .02$ & $\mathbf{92.30 \pm .24}$\\
Last layer (\texttt{{Civil}}) & $63.1 \pm .12$ &  $64.7 \pm .11$ & $\mathbf{65.8 \pm .12}$\\
Full model (\texttt{{Civil}}) & $65.8 \pm .01$&  $67.2 \pm .01$ & $\mathbf{68.3 \pm .14}$\\
%\bottomrule
\end{tabular}
\end{center}
\label{tbl:fine-tuning}
\end{table}

Here, we applied our algorithms to fine-tuning pre-trained models in
classification. In the pre-training/fine-tuning paradigm
\citep{raffel2019exploring}, a model is pre-trained on a generalist
dataset (coming from $\sQ$). The model is then fine-tuned on a
task-specific dataset (generated from $\sP$). Two predominantly used
fine-tuning approaches are {\em last-layer fine-tuning}
\citep{subramanian2018learning, kiros2015skip} and {\em full-model
  fine-tuning} \citep{howard-ruder-2018-universal}. In the former, the
representations obtained from the last layer of the pre-trained model
are used to train a simple model (often a linear hypothesis) on the
data from $\sP$. We chose the simple model to be a multi-class
logistic regression model. In the latter approach, the model is
initialized from the pre-trained model and all the parameters are
fine-tuned (often via gradient descent) on $\sP$. We explored the
additional advantages of combining data from both $\sP$ and $\sQ$
during fine-tuning.  There has been recent interest in carefully
combining various tasks/data for the purpose of fine-tuning and avoid
the phenomenon of ``negative transfer'' \citep{aribandi2021ext5}. Our
proposed theory presents a principled approach.
%towards that objective.

We used the CIFAR-10 vision dataset \citep{krizhevsky2009learning} and
formed a pre-training task (source) by combining data from classes:
\{'airplane', 'automobile', 'bird', 'cat', 'deer', 'dog'\}. For this
task we use a standard ResNet-18 architecture \citep{he2016deep}. The
fine-tuning task (target) consists of data from classes: \{'frog',
'horse', 'ship', 'truck'\}. In addition, we also used the {\tt \small
  {Civil Comments}} dataset. For this we used a BERT-small model
\citep{devlin2018bert} for pre-training.  For more detail on the
dataset and experimental procedure, see
Appendix~\ref{app:exp-details}.
\ignore{
  We compared the standard approach of only using data from $\sP$
  against our proposed algorithms and gapBoost. We took $60\%$ of the
  data from the source for pre-training, and the remaining $40\%$ was
  used in fine-tuning. We split the fine-tuning data randomly into a
  $70\%$ training set to be used in fine-tuning, and the remaining
  $30\%$ to be used as a test set. The results are reported over $5$
  random splits. Pre-training is done on a standard ResNet-18
  architecture \citep{he2016deep} by optimizing the cross-entropy
  loss.
}
As can be seen from Table~\ref{tbl:fine-tuning}, \sbest\ comfortably
outperforms both the standard approach of training just on $\sP$, as
well as gapBoost.

\section{Conclusion}
\label{sec:conclusions}

We presented a comprehensive study of best-effort adaptation (or
supervised adaptation), including a new discrepancy-based theoretical
analysis, algorithms benefiting from the corresponding learning
guarantees, as well as a series of empirical results demonstrating the
performance of these algorithms in several tasks.  We further showed
how our analysis can be leveraged to derive learning guarantees in
domain adaptation, as well as new enhanced adaptation algorithms.  Our
analysis and algorithms are likely to be useful in the study of other
adaptation scenarios and admit a variety of other applications. In
fact, our analysis applies to any sample reweighting method.

\ignore{
\section*{Data availability statement}

The datasets analyzed in this study are all public datasets and are
available from the URLs referenced. Our artificial dataset used for a
simulation is described in detail and the code generating it can be
provided upon request.
}

\section*{Acknowledgments}

We thank Jamie Morgenstern for several discussions about this work at
Google Research.

\newpage
\bibliography{abeff,aux}

\newpage

%\begin{appendices}

\appendix
\onecolumn

\renewcommand{\contentsname}{Contents of Appendix}
\tableofcontents
\addtocontents{toc}{\protect\setcounter{tocdepth}{3}} 
\clearpage

\section{Related work}
\label{app:related}

\subsection{Adaptation and transfer learning}
\label{app:related-adap-transfer}

\textbf{Discrepancy-based adaptation theory.}  The work we present
includes a significant theoretical component and benefits from prior
theoretical analyses of domain adaptation.  The theoretical analysis
of domain adaptation was initiated by \cite{KiferBenDavidGehrke2004}
and \cite{BenDavidBlitzerCrammerPereira2006} with the introduction of
a \emph{$d_A$-distance} between distributions. They used this notion
to derive VC-dimension learning bounds for the zero-one loss, which
was elaborated on in follow-up publications like
\citep{BlitzerCrammerKuleszaPereiraWortman2008,
  BenDavidBlitzerCrammerKuleszaPereiraVaughan2010}.  Later,
\cite{MansourMohriRostamizadeh2009} and
\cite{CortesMohri2011,CortesMohri2014} presented a general analysis
of single-source adaptation for arbitrary loss functions, where they
introduced the notion of \emph{discrepancy}, which they argued is a
divergence measure tailored to domain adaptation.  The notion of
discrepancy coincides with the $d_A$-distance in the special case of
the zero-one loss. It takes into account the loss function and the
hypothesis set and, importantly, can be estimated from finite samples.
The authors further gave Rademacher complexity learning bounds in
terms of the discrepancy for arbitrary hypothesis sets and loss
functions, as well as pointwise learning bounds for kernel-based
hypothesis sets. They also gave a discrepancy minimization algorithm
based on a reweighting of the losses of sample points.  We use their
notion of discrepancy in our new analysis.
\cite{CortesMohriMunozMedina2019} presented an extension of the
discrepancy minimization algorithm based on the so-called
\emph{generalized discrepancy}, which allows for the weights to be
hypothesis-dependent and which works with a less conservative notion
of \emph{local discrepancy} defined by a supremum over a subset of the
hypothesis set. The notion of local discrepancy has been since adopted
in several recent publications, in the study of active learning or
adaptation \citep{DeMathelinMougeotVayatis2021,ZhangLiuLongJordan2019,
  ZhangLongWangJordan2020} and is also used in part of our analysis.
Finally, a PAC-Bayesian analysis of adaptation has also been given by
\cite{GermainHabrardLavioletteMorvant2013}, using a related notion of
discrepancy.  Note also that, as argued in
Appendix~\ref{app:disc-est}, for our analysis of best-effort
adaptation and algorithms, we can restrict ourselves to a small ball
$\sfB(h_\sP, r)$ around the best hypothesis found by training on
$\sP$, with $r$ in the order of $1/\sqrt{n}$.  This leads to a more
favorable discrepancy term, which is similar to the \emph{super
transfer} or \emph{localization} benefits mentioned by
\cite{HannekeKpotufe2019}. This advantage can be leveraged when there
is a sufficient amount of labeled data from the target distribution,
as in the scenario of best-effort adaptation. In standard domain
adaptation, however, it would not be possible to estimate such local
discrepancy quantities, which are also used in the analysis of
\cite{ZhangLongWangJordan2020}, and thus the corresponding learning
bounds or notions would be not be algorithmically useful.

A theoretical analysis and algorithm for driting distributions are
given by \cite{MohriMunozMedina2012}.
The assumptions made in the analysis of adaptation were discussed by
\cite{BenDavidLuLuuPal2010} who presented several negative results for
the zero-one loss.

Many of the theoretical guarantees for domain adaptation
\citep{BenDavidBlitzerCrammerPereira2006,
  BenDavidBlitzerCrammerKuleszaPereiraVaughan2010,v97-zhang19i}
have upper bounds that include the term $\lambda_\sH = \min_{h \in
  \sH} \curl*{\cL(\sP, h) + \cL(\sQ, h)}$, which, as pointed out by
\cite{MansourMohriRostamizadeh2009}, roughly doubles the
representation error one incurs for $\sH$ and results overall in
learning bounds with a factor of $3$ of the error with the respect to
an ideal target. This can make these bounds vacuous in some natural
scenarios. Moreover, the $\lambda_\sH$ terms cannot be estimated from
observations.  The learning bounds of
\cite{MansourMohriRostamizadeh2009} do not admit the factor of $3$ of
the error drawback, but they also contain terms depending on the
best-in-class predictors with respect to both distributions that
cannot be estimated.  In general, they are not comparable with the
bounds of \cite{BenDavidBlitzerCrammerPereira2006}.
Our learning bounds differ from these analyses since we compare the
target loss of a predictor with an empirical $\sfq$-weighted empirical
loss on a sample from $\sQ$ or both $\sQ$ and $\sP$ and not just with
an unweighted loss for a sample drawn from $\sQ$.  Furthermore, our
learning guarantees are high-probability bounds, while those of these
previous work hold with probability one. The latter can be derived
from straightforward applications of triangle inequality.  Crucially,
our learning bounds can be leveraged by algorithms, while previous
bounds do not include any non-trivial term that can be optimized.

\textbf{Multiple-source adaptation theory.}
\cite{MansourMohriRoSureshWu2021} presented a theory of
multiple-source adaptation with limited target labeled data using the
notion of discrepancy. A series of publications by
\cite{MansourMohriRostamizadeh2009,MansourMohriRostamizadeh2009a},
\cite{HoffmanMohriZhang2018,HoffmanMohriZhang2021,HoffmanMohriZhang2022}
and \cite{CortesMohriSureshZhang2021} give an extensive theoretical
and algorithmic analysis of the problem of \emph{multiple-source
adaptation} (MSA) scenario where the learner has access to unlabeled
samples and a trained predictor for each source domain, with no access
to source labeled data.  This approach has been further used in many
applications such as object recognition \citep{hoffman_eccv12,
  gong_icml13, gong_nips13}.  \cite{zhao2018adversarial} and
\cite{wen2019domain} considered MSA with only unlabeled target data
available and provided generalization bounds for classification and
regression.

\textbf{Other adaptation analyses.} There are alternative analyses of
the adaptation problem based on divergences between distributions that
do not take into account the specific loss function or hypothesis set
used. These include methods based on importance weighting
\citep{SugiyamaEtAl2007a,ZhangYamaneLuSugiyama2020,LuZhangFangTeshimaSugiyama2021,SugiyamaKrauledatMuller2007}.
\cite{CortesMansourMohri2010} gave a theoretical analysis of
importance weighting, including learning bounds based on the analysis
of unbounded loss functions (see also
\citep{CortesGreenbergMohri2019}), showing both theoretically and
empirically that importance weighting can fail in a number of cases,
depending on the magnitude of the second-moment of the weights,
including in simple cases of the two domain being Gaussian
distributions. This holds even for perfectly estimated importance
weights. The publications in this category also include those using
the Wasserstein distance
\citep{CourtyFlamaryTuiaRakotomamonjy2017,\ignore{CourtyFlamaryHabrardRakotomamonjy2017,}RedkoHabrardSebban2017}, which in some sense is closer to the notion
of discrepancy but yet does not capture the hypothesis set used.  An
alternative distance used is that of Kernel Mean Matching (KMM), which
is the difference between the expectation of the feature vector in the
source domain and the target domain
\citep{HuangSmolaGrettonBorgwardtScholkopf2006}. Several other
publications have also adopted also that distance
\citep{LongCaoWangJordan2015,RedkoBennani2016}. The KMM algorithm
seeks to reweight the source sample to make this difference as small
as possible. This, however, ignores other moments of the
distributions, as well as the loss function and the hypothesis
sets. Nevertheless, in some instances, the distance is close to and
somewhat related to discrepancy. The experiments reported by
\cite{CortesMohri2014} suggest that, while in some instances KMM
performs well, in some others it does not. This variance might be due
to the fact that the distance does not always capture key aspects
related to the loss function and the hypothesis set. In other
experiments reported by \cite{CortesMohriMunozMedina2019}, the
performance of KMM is sometimes worse than training on the sample $S$
drawn from $\sQ$ (without reweighting). This problem was already
reported for another algorithm, KLIEP, by
\cite{SugiyamaEtAl2007a}\ignore{, which can
  also be viewed as a two-stage reweighting method}.
Variants of boosting designed for transfer also tacitly reweight
examples~\citep{huang2017cross,zheng2020improved}.

Note that the algorithms suggested for KMM, importance-weighting,
KLIEP and other similar methods can all be viewed as specific methods
for reweighting the sample losses. In that sense, they are all covered
by our general analysis, when the weights are bounded. However, note
also that they are all two-stage algorithms: the weights are first
chosen to reduce or minimize some distance, irrespective of their
effect on the weighted empirical loss, and next the weights are fixed
and used to minimize the empirical weighted loss.

An interesting non-parametric analysis of adaptation is presented in
\citep{KpotufeMartinet2018,HannekeKpotufe2019}. \cite{HannekeKpotufe2019}
do not give an adaptation algorithm, however. A causal view of
adaptation is also analyzed in
\citep{ZhangScholkopfMuandetWang2013,GongZhangLiuTaoGlymourScholkopf2016}.

\textbf{Transfer learning analyses.}  Other scenarios of transfer
learning have been studied by
\cite{KuzborskijOrabona2013,PerrotHabrard2015,DuKoushikSinghPoczos2017}
including by leveraging smaller target labeled data and auxiliary
hypotheses (see also \citep{HannekeKpotufe2019} already mentioned).
The problem of active adaptation or transfer learning has been
investigated by several publications \cite{YangHannekeCarbonell2013,
  ChattopadhyayFanDavidsonPanchanathanYe2013,BerlindUrner2015}. Another
somewhat related problem is that of multi-task learning studied by
\cite{Maurer2006,MaurerPontilRomeraParedes2016,
  PentinaLampert2017,PentinaBenDavid2018}.  The scenario of life-long
learning is also somewhat related
\citep{PentinaLampert2014,PentinaLampert2015,PentinaUrner2016,
  BalcanKhodakTalwalkar2019}.

\textbf{Other adaptation or transfer learning publications.}  The
space of transfer learning and domain adaptation approaches is
massive~\citep{chen2011co,zhang2019bridging,wang2011heterogeneous,
  sener2016learning,hoffman_eccv12,ghifary2016deep,zhao2019learning,zhao2018adversarial,
  li2018transfer,bousmalis2017unsupervised,sun2016return,kundu2020universal,
  sun2016deep,ghifary2016scatter,long2016unsupervised,courty2016optimal,
  saito2018maximum,wang2018visual,motiian2017few,sun2016deep}
  and includes interesting analyses 
  and observations such as that of
  \cite{daume2007frustratingly} 
  about a surprisingly good baseline
  and
  follow-up by \cite{sun2016return}. We recommend
readers to surveys such as~\cite{pan2009,wang2018deep,
  li2012literature} for a comprehensive overview. We briefly outline
the most relevant approaches here.

There is a very large recent literature dealing with experimental
studies of domain adaptation in various tasks. \cite{ganin2016domain}
proposed to learn features that cannot discriminate between source and
target domains. \cite{tzeng2015simultaneous} proposed a CNN
architecture to exploit unlabeled and sparsely labeled target domain
data. \cite{motiian2017unified}, \cite{motiian2017few} and
\cite{wang2019few} proposed to train maximally separated features via
adversarial learning. \cite{saito2019semi} proposed to use a minmax
entropy method for domain adaptation.

Several algorithms have been proposed for multiple-source
adaptation. \cite{khosla2012undoing,blanchard2011generalizing}
proposed to combine all the source data and train a single
model. \cite{duan2009domain,duan2012domain} used unlabeled target data
to obtain a regularizer. Domain adaptation via adversarial learning
was studied by \cite{multiadversarial_aaai2018,zhao2018adversarial}.
\cite{crammer2008learning} considered learning models for each source
domain, using close-by data of other domains.  \cite{gong_cvpr12}
ranked multiple source domains by how well they can adapt to a target
domain. Other solutions to multiple-source domain adaptation include,
clustering \citep{liu2016structure}, learning domain-invariant
features \citep{gong_icml13}, learning intermediate representations
\citep{jhuo2012robust}, subspace alignment techniques
\citep{fernando2013unsupervised}, attributes detection
\citep{gan2016learning}, using a linear combination of pre-trained
classifiers \citep{yang_acmm07}, using multitask auto-encoders
\citep{ghifary2015domain}, causal approaches \citep{sun2011two},
two-state weighting approaches \citep{sun2011two}, moments alignment
techniques \citep{peng2019moment} and domain-invariant component
analysis \citep{MuandetBalduzziScholkopf2013}.

When some labeled data from both source and target are available, a
variety of practical methods have been studied.
\cite{daume2007frustratingly} performs an empirical comparison
amongst a collection of basic models when some labeled data is
available from both source and target: source-only, target-only,
training on all data together, uniformly $\alpha$-weighting the source
data and $(1-\alpha)$-weighting the target data, using the prediction
of a model on the source as a feature for training on the target,
linearly interpolating between source-only and target-only models, and
a ``lifted" approach where each sample is projected into $\sX^3$,
corresponding to source/target/general information copies of the
feature space, and show empirically that each of these benchmarks
performs fairly well, with the latter outperforming the others most of
the time.

Some recent work focuses on adversarial adaptation
\citep{motiian2017few,multiadversarial_aaai2018,ganin2016domain}.
The problem of \emph{domain generalization}, that is generalization
to an arbitrary target distribution within some set has been
studied by \citep{mohri2019agnostic} and is also related to
that of robust learning
\citep{chen2017robust,konstantinov2019robust,jhuo2012robust}.

We discuss separately, in the following section, the relationship of
our work with fine-tuning methods.

\subsection{Relationship with fine-tuning methods}
\label{app:fine-tuning}

Here, we discuss the connection of our work with fine-tuning
\citep{howard-ruder-2018-universal, peters-etal-2018-deep,
  Houlsby-et-al} of pre-trained models. A comprehensive description of
fine-tuning methods is beyond the scope of this work, but see
\citep{Guo_2019_CVPR,you2020co, aribandi2021ext5,
  aghajanyan2021muppet, wei2021finetuned} for some recent results. A
related area is few shot-learning algorithms and related meta-learning
algorithms such as MAML \citep{FinnAbbeelLevine2017} include
\citep{wang2019few,motiian2017few}, and Reptile
\citep{NicholAchiamSchulman2018}.  \ignore{ This web site could be
  useful for more pointers:
  https://boyangzhao.github.io/posts/few_shot_learning#reptile }

In general, consider a scenario where there exists good common feature
mapping $\Phi\colon \sX \to \Rset^d$ for both the $\sQ$ and $\sP$. Let
$f$ be the result of pre-training a neural network on $\sQ$ data.  The
mapping in $f$ corresponding to some depth of the hidden layers can
then be viewed as a good approximation of $\Phi$. Alternatively,
$\Phi$ may be the output of a representation learning algorithm.

There are several fine-tuning methods introduced in the literature
\citep{subramanian2018learning, kiros2015skip,
  howard-ruder-2018-universal, raffel2019exploring} that consists of
adapting $f$ to domain $\sP$. This may be by using $f$ as an
initialization point and applying SGD with sample $S'$ drawn from
$\sP$, while fixing the hidden layer parameters to a given depth. It
may be by \emph{forgetting} the weights at the top layer(s) and
retraining them by using $S'$ alone. Or, it may be done by continuing
training with a mixture of $S'$ and a new sample from $S$. Training on
such a mixture avoids `catastrophic forgetting'. In all cases, the
problem can be cast as that of learning a hypothesis with feature
vector $\Phi$ by using sample $S$ and $S'$, or sample $S'$ alone,
which is a special case of the scenario we analyzed in
Section~\ref{sec:disc-theory}. The algorithms presented in
Section~\ref{sec:disc-algorithms} provide a principled solution to
this problem by taking into consideration the discrepancy between
$\sQ$ and $\sP$ and by selecting suitable $\sfq$-weights to guarantee
a better generalization.

\newpage
\section{Best-effort adaptation}
\label{app:qgen}

\subsection{Theorems and proofs} 

Below we will work with a generalized notion of discrepancy as defined
in \eqref{eq:disc-definition-app}. Given distributions $\sP, \sQ$ and
positive real numbers $a, b$ we define the weighted discrepancy as
\begin{equation}
\label{eq:disc-definition-app}
  \dis(a\sP, b\sQ)
  =
  \sup_{h \in \sH} 
    \E_{\substack{(x, y) \sim \sP}} \bracket*{a \cdot \ell(h(x), y)}
    -     \E_{\substack{(x, y) \sim \sQ}} \bracket*{b \cdot \ell(h(x), y)}.
\end{equation}
\qGeneralTheorem*

\begin{proof}
  Let $S = ((x_1, y_1), \ldots, (x_m, y_m))$ be a sample of size $m$
  drawn i.i.d.\ from $\sQ$ and similarly
  $S' = ((x_{m + 1}, y_{m + 1}), \ldots, (x_{m + n}, y_{m + n}))$ a
  sample of size $n$ drawn i.i.d.\ from $\sP$.  Let $T$ denote the
  sample formed by $S$ and $S'$, $T = (S, S')$.  For any such sample
  $T$, let $\Phi(T)$ denote
  $\Phi(T) = \sup_{h \in \sH} \cL(\ov \sfq \sQ + (\norm{\sfq}_1 - \ov \sfq) \sP, h) -
  \cL(\sfq, h)$, with
  $\cL(\sfq, h) = \sum_{i = 1}^{m + n} \sfq_i \ell(h(x_i),
  y_i)$. Changing point $x_i$ to some other point $x'_i$ affects
  $\Phi(T)$ by at most $\sfq_i$. Thus, by McDiarmid's inequality, for
  any $\delta > 0$, with probability at least $1 - \delta$, the
  following holds for all $h \in \sH$:
  \begin{equation}
  \label{eq:eq1}
\cL(\ov \sfq \sQ + (\norm{\sfq}_1 - \ov \sfq) \sP, h) \leq \cL(\sfq, h) + \E[\Phi(T)] 
+ \| \sfq \|_2 \sqrt{\frac{\log \frac{1}{\delta}}{2}}.
\end{equation}
We now analyze the expectation term:
\begin{align*}
\E[\Phi(T)]
& = \E_T\bracket*{\sup_{h \in \sH} \cL(\ov \sfq \sQ + (\norm{\sfq}_1 - \ov \sfq) \sP, h) 
- \cL_T(\sfq, h)}\\
& = \E_T\bracket*{\sup_{h \in \sH} \E_{T'} \bracket*{\cL_{T'}(\sfq, h) - \cL_T(\sfq, h)}}\\
& \leq \E_{T, T'} \bracket*{\sup_{h \in \sH} \cL_{T'}(\sfq, h) - \cL_T(\sfq, h)}\\
& = \E_{T, T'} \bracket*{\sup_{h \in \sH} \sum_{i = 1}^{m + n} \sfq_i \ell(h(x'_i), y'_i) - 
\sfq_i \ell(h(x_i), y_i)}\\
& = \E_{T, T', \bsigma} \bracket*{\sup_{h \in \sH} \sum_{i = 1}^{m + n} \sigma_i \paren*{\sfq_i \ell(h(x'_i), y'_i) - 
\sfq_i \ell(h(x_i), y_i)}}\\
& \leq 2 \E_{T, \bsigma} \bracket*{\sup_{h \in \sH} 
\sum_{i = 1}^{m + n} \sigma_i \sfq_i \ell(h(x_i), y_i)}
= 2 \Rad_{\sfq}(\ell \circ \sH).
\end{align*}
We make a remark about the validity of the second equality in the
above derivation. Let $T'$ be a sample that has the same distribution
as $T$. Furthermore we will use $(x'_i, y'_i)$ to denote the $i$th
sample in $T'$. Then notice that
\begin{align}
    \E_{T'}[L_{T'}(q, h)] &= \sum_{i=1}^m q_i \E[\ell(h(x'_i), y_i)] + \sum_{i=m+1}^{m+n} q_i \E[\ell(h(x'_i), y_i)]\\
    &= \sum_{i=1}^m q_i \cL(\sQ, h) + \sum_{i=m+1}^{m+n} q_i \cL(\sP, h)\\
    &= \ov \sfq \cL(\sQ, h) + (\norm{\sfq}_1 - \ov \sfq)\cL(\sP, h)\\
    &= \cL(\ov \sfq \sQ + (\norm{\sfq}_1 - \ov \sfq) \sP, h)
\end{align}
Finally, using the upper bound
$\cL(\sP, h) - \cL(\ov \sfq \sQ + (\norm{\sfq}_1 - \ov \sfq) \sP, h)
\leq \dis \paren*{\sP, \ov \sfq \sQ + (\norm{\sfq}_1 - \ov \sfq) \sP} = 
\dis\paren*{\bracket*{(1 - \norm{\sfq}_1) + \ov \sfq} \sP, \ov \sfq \sQ}$
completes the proof.
\end{proof}

Next, we show that the bound above is tight in terms of the
weighted-discrepancy term.

\LowerBound*
\begin{proof}
  Let $\cL(\sfq, h)$ denote $\sum_{i = 1}^{m + n} \sfq_i \ell(h(x_i),
  y_i)$.  By definition of discrepancy as a supremum, for any $\e >
  0$, there exists $h \in \sH$ such that $\cL(\sP, h) - \cL(\sQ, h)
  \geq \dis(\sP, \sQ) - \e$.  For that $h$, we have
  \begin{align*}
    \cL(\sP, h) - \ov \sfq \dis(\sP, \sQ) - \cL(\sfq, h) 
    & \geq \cL(\sP, h) - \ov \sfq \paren*{\cL(\sP, h) - \cL(\sQ, h)} - \cL(\sfq, h) - \e\\
    & = (1 - \ov \sfq) \cL(\sP, h) + \ov \sfq \cL(\sQ, h) - \cL(\sfq, h)  - \e\\
    & = \E[\cL(\sfq, h)] - \cL(\sfq, h)  - \e.
  \end{align*}
  By McDiarmid's inequality, with probability at least $1 - \delta$, we have
  $\E[\cL(\sfq, h)] - \cL(\sfq, h)
  \geq
  - 2 \Rad_{\sfq}(\ell \circ \sH)
  - \| \sfq \|_2 \sqrt{\frac{\log \frac{1}{\delta}}{2}}$. Thus, we have:
  \[
  \cL(\sP, h) - \ov \sfq \dis(\sP, \sQ) - \cL(\sfq, h)
  \geq - 2 \Rad_{\sfq}(\ell \circ \sH)
  - \| \sfq \|_2 \sqrt{\frac{\log \frac{1}{\delta}}{2}}
  - \e.
  \]
  The last inequality follows directly by using the assumptions
  and Lemma~\ref{lemma:rad}.
\end{proof}

\qGeneralTheoremUniform*

\begin{proof}
Consider two sequences $(\e_k)_{k \geq 0}$ and $(\sfq^k)_{k \geq
  0}$. By Theorem~\ref{th:qgen}, for any fixed $k \geq 0$, we have:
\begin{multline*}
\Pr \Bigg[ \cL(\sP, h)
  > \sum_{i = 1}^{m + n} \sfq^k_i \ell(h(x_i), y_i)
+ \dis\paren*{\bracket*{(1 - \norm{\sfq^k}_1) + \ov \sfq^k} \sP, \ov \sfq^k \sQ}\\
+  2 \Rad_{\sfq^k}(\ell \circ \sH)
+ \frac{\| \sfq^k \|_2}{\sqrt{2}} \e_k \Bigg]
\leq e^{-\e_k^2}.
\end{multline*}
Choose $\e_k = \e + \sqrt{2 \log (k + 1)}$. Then, by the union bound,
we can write:
\begin{align}
\label{eq:1}
\Pr \Bigg[
& \exists k \geq 1 \colon \cL(\sP, h)
  > \sum_{i = 1}^{m + n} \sfq^k_i \ell(h(x_i), y_i)
+ \dis\paren*{\bracket*{(1 - \norm{\sfq^k}_1) + \ov \sfq^k} \sP, \ov \sfq^k \sQ} \\
& +  2 \Rad_{\sfq^k}(\ell \circ \sH)
+ \frac{\| \sfq^k \|_2}{\sqrt{2}} \e_k \Bigg] \nonumber \\
& \leq \sum_{k = 0}^{+\infty} e^{-\e_k^2}
\leq \sum_{k = 0}^{+\infty} e^{-\e^2 - \log ((k + 1)^2)}
= e^{-\e^2} \sum_{k = 1}^{+\infty} \frac{1}{k^2}
= \frac{\pi^2}{6} e^{-\e^2}
\leq 2 e^{-\e^2}.\nonumber 
\end{align}
We can choose $\sfq^k$ such that $\| \sfq^k - \sfp^0 \|_1 = 1 -
\frac{1}{2^k}$.  Then, for any $\sfq \in \curl*{\sfq \colon 0 \leq \|
  \sfq - \sfp^0 \|_1 < 1}$, there exists $k \geq 0$ such that $\|
\sfq^k - \sfp^0 \|_1 \leq \| \sfq - \sfp^0 \|_1 < \| \sfq^{k + 1} -
\sfp^0 \|_1$ and thus such that
\begin{align*}
\sqrt{2 \log (k + 1)} 
= \sqrt{2 \log \log_2 \frac{1}{1 - \norm{\sfq^{k + 1} - \sfp^0}_1}}
& = \sqrt{2 \log \log_2 \frac{2}{1 - \norm{\sfq^k - \sfp^0}_1}}\\
& \leq \sqrt{2 \log \log_2 \frac{2}{1 - \norm{\sfq - \sfp^0}_1}}.
\end{align*}
Furthermore, for that $k$, the following inequalities hold:
\begin{align*}
\sum_{i = 1}^{m + n} \sfq^k_i \ell(h(x_i), y_i) 
& \leq \sum_{i = 1}^{m + n} \sfq_i \ell(h(x_i), y_i) + \dis(\sfq^k, \sfq)\\
& \leq \sum_{i = 1}^{m + n} \sfq_i \ell(h(x_i), y_i) + \dis(\sfq^k, \sfp^0) + \dis(\sfp^0, \sfq)\\
& \leq \sum_{i = 1}^{m + n} \sfq_i \ell(h(x_i), y_i) + \norm{\sfq^k - \sfp^0}_1 + \dis(\sfq, \sfp^0)\\
& \leq \sum_{i = 1}^{m + n} \sfq_i \ell(h(x_i), y_i) + \norm{\sfq - \sfp^0}_1 + \dis(\sfq, \sfp^0),\\
  \dis\paren*{\bracket*{(1 - \norm{\sfq^k}_1) + \ov \sfq^k} \sP, \ov \sfq^k \sQ}
    & \leq \dis(\bracket*{(1 - \norm{\sfq}_1) + \ov \sfq} \sP, \ov \sfq \sQ)\\
    & \qquad + \norm*{ \bracket*{(\norm{\sfq}_1 - \ov \sfq) - (\norm{\sfq^k}_1 - \ov \sfq^k)} \sP + \bracket*{\ov \sfq - \ov \sfq^k} \sQ }_1\\
%  & = \paren*{\ov \sfq + \ov \sfq^k - \ov \sfp^0 + \ov \sfp^0 - \ov \sfq} \dis(\sP, \sQ) 
  & \leq \dis(\bracket*{(1 - \norm{\sfq}_1) + \ov \sfq} \sP, \ov \sfq \sQ) + \norm{\sfq^k -  \sfq}_1\\
  & \leq \dis(\bracket*{(1 - \norm{\sfq}_1) + \ov \sfq} \sP, \ov \sfq \sQ) + 2 \norm{\sfq - \sfp^0}_1,\\
%  \norm{\sfq^k - \ov \sfp^0}_1 + 
%  (\ov \sfq^k - \ov \sfp^0 + \ov \sfp^0) \dis(\sP, \sQ)
%\leq \norm{\sfq^k - \ov \sfp^0}_1 + \ov \sfp^0 \dis(\sP, \sQ)
  %\leq \norm{\sfq - \ov \sfp^0}_1 + \ov \sfp^0 \dis(\sP, \sQ),\\
  \Rad_{\sfq^k}(\ell \circ \sH)
  & \leq \Rad_{\sfq}(\ell \circ \sH) + \norm{\sfq^k -  \sfq}_1
  \leq \Rad_{\sfq}(\ell \circ \sH) + 2 \norm{\sfq - \sfp^0}_1,\\
  \text{and} \qquad \norm{\sfq^k}_2 
 & \leq \norm{\sfq}_2 + \norm{\sfq^k - \sfq}_2\\
 &  \leq \norm{\sfq}_2 + \norm{\sfq^k - \sfq}_1
  \leq \norm{\sfq}_2 + 2 \norm{\sfq - \sfp^0}_1.
% & \leq \norm{\sfq^k - \sfp^0}_2 + \norm{\sfp^0}_2
%   \leq \norm{\sfq^k - \sfp^0}_1 + \norm{\sfp^0}_2
%   \leq \norm{\sfq - \sfp^0}_1 + \norm{\sfp^0}_2.
\end{align*}
Plugging in these inequalities in \eqref{eq:1} concludes the proof.
\end{proof}

\qGeneralCorollary*

\begin{proof}
Note that the discrepancy term of the bound of
Theorem~\ref{th:qgen-uniform} can be further upper bounded as follows:
\begin{align*}
&  \dis\paren*{\bracket*{(1 - \norm{\sfq}_1) + \ov \sfq} \sP, \ov \sfq \sQ}\\
& = \sup_{h \in \sH} \curl*{
  \bracket*{(1 - \norm{\sfq}_1) + \ov \sfq} \E_{(x, y) \sim \sP}[\ell(h(x), y)]
  - \ov \sfq \E_{(x, y) \sim \sQ}[\ell(h(x), y)]}\\
& \leq \ov \sfq \dis(\sP, \sQ) + \abs[\big]{1 - \norm{\sfq}_1} \sup_{h \in \sH} \E_{(x, y) \sim \sP}[\ell(h(x), y)]\\
& \leq \ov \sfq \dis(\sP, \sQ) + \abs[\big]{1 - \norm{\sfq}_1} \\
& = \ov \sfq \dis(\sP, \sQ) + \abs[\big]{\norm{\sfp^0}_1 - \norm{\sfq}_1} \\  
& \leq \ov \sfq \dis(\sP, \sQ) + \norm{\sfp^0 - \sfq}_1.
\end{align*}
Plugging this in the right-hand side in the bound of
Theorem~\ref{th:qgen-uniform} completes the proof.
\end{proof}

\begin{restatable}{lemma}{RademacherLemma}
\label{lemma:rad}
Fix a distribution $\sfq$ over $[m + n]$. Then, the following %upper
%bound 
holds for the $\sfq$-weighted Rademacher complexity:
\[
\Rad_{\sfq}(\ell \circ \sH)
\leq \| \sfq \|_\infty (m + n) \, \Rad_{m + n}(\ell \circ \sH).
\]
\end{restatable}

%\RademacherLemma*

\begin{proof}
Since for any $i \in [m + n]$, the function $\varphi_i \colon x
\mapsto \sfq_i x$ is $\sfq_i$-Lipschitz and thus
$\norm{\sfq}_\infty$-Lipschitz, the result is an application of
Talagrand's inequality \citep{LedouxTalagrand1991}.
\end{proof}
Note that the bound of the lemma is tight: equality holds
when $\sfq$ is chosen to be the uniform distribution.
By McDiarmid's inequality, the $\sfq$-weighted Rademacher complexity
can be estimated from the empirical quantity
\[
\h \Rad_{\sfq, S, S'}(\ell \circ \sH) = \E_{\bsigma} \bracket*{\sup_{h \in \sH}
\sum_{i = 1}^{m + n} \sigma_i \sfq_i \ell(h(x_i), y_i)},
\]
modulo a term in $O(\| \sfq \|_2)$. 

\newpage
\subsection{Discussion of learning bound of Theorem~\ref{th:qgen}}
\label{app:discussion}

It is instructive to examine some special cases for the choice of
$\sfq$, which will demonstrate how our guarantees can recover
several previous bounds as a special case.
Since our algorithms seek to choose the best weight (and best
hypothesis) based on these bounds, this shows that their search
space includes that of algorithms based on those previous bounds.
\ignore{
That will demonstrate that our bound is a finer one than
existing discrepancy bounds.  It will also illustrate why our
algorithms, which make effective use of this new bound, will be able
to outperform those based on previous bounds.
}

\textbf{$\sfq$ chosen uniformly on $S$.}  For $\sfq$ chosen to be the
uniform distribution on $S$, we have $\ov \sfq = 1$, $\norm{\sfq}_2 =
\frac{1}{\sqrt{m}}$, and the bound coincides with the labeled
discrepancy-based bound for $\sP$ of
\cite{CortesMohriMunozMedina2019}[Prop.\ 5; Eq.\ (9)].
Indeed, for $\sfq$ chosen to be supported only on $S$, the theorem
gives a $\sfq$-discrepancy domain adaptation bound from $\sQ$ to
$\sP$, in terms of a $\sfq$-Rademacher complexity and $\| \sfq \|_2$.

\textbf{$\sfq$ chosen uniformly on $S'$.}
%For $\sfq$ chosen to be the uniform distribution over $S'$, 
%In this case, we have
Here $\ov \sfq = 0$, $\norm{\sfq}_2 = \frac{1}{\sqrt{n}}$, and the bound
coincides with the standard Rademacher complexity bound for $\sP$
for learning from a labeled sample of size $n$:
\begin{equation}
\label{eq:P-bound}
\cL(\sP, h)
\leq
\frac{1}{n}
\sum_{i = m + 1}^{m + n}
\ell(h(x_i), y_i)
+  2 \Rad_{n}(\ell \circ \sH)
+
 \sqrt{\frac{\log \frac{1}{\delta}}{2n}}.
\end{equation}
Here, $\Rad_{n}(\ell \circ \sH)$ is the standard Rademacher complexity
defined as in \eqref{eq:q-weighted-Rad} where the expectation is over
$S'$ and $\sfq$ is the uniform distribution over $S'$.
Thus, for $\sfq$ minimizing the right-hand side of the
bound of the theorem, the learning bound is at least as favorable as
one restricted to learning from the labeled points from $\sP$.
But the bound also demonstrates that it is possible to do better than
just learning from $\sP$.  In fact, for $\sQ = \sP$, we have
$\dis(\sP, \sQ) = 0$, and $\sfq$ can be chosen to be uniform over $T =
(S, S')$, thus $\norm{\sfq}_2 = \frac{1}{\sqrt{m + n}}$. The bound
then coincides with the standard Rademacher complexity bound for a
sample of size $m + n$ for the distribution $\sP$.  More generally,
such a bound holds for any two distributions $\sP$ and $\sQ$ with
$\dis(\sP, \sQ) = 0$.

The learning bound \eqref{eq:P-bound} can be straightforwardly
upper-bounded by the weighted discrepancy bound of
\cite{CortesMohriMunozMedina2019}[Prop.~5; Eq.~(10)], for any $\sfp$
with support $S$:
\begin{align}
\label{eq:P-bound2}
\cL(\sP, h)
\leq 
\sum_{i = 1}^m \sfp_i \ell(h(x_i), y_i)
+ \dis(\h \sP, \sfp)
+  2 \Rad_{n}(\ell \circ \sH)
+ \bracket*{
  \frac{\log \frac{1}{\delta}}{2n}}^{\frac{1}{2}
  },
\end{align}
using the inequality
\[
\cL(\h \sP, h) \leq \sum_{i = 1}^m \sfp_i
\ell(h(x_i), y_i) + \dis(\h \sP, \sfp),
\]
which holds for any $\sfp$, by definition of the discrepancy. Thus,
there is a specific choice of the weights in our bound that makes it a
lower bound for that of \cite{CortesMohriMunozMedina2019}, regardless
of how the weights $\sfp$ are chosen in their bound (the inequality
holds uniformly over $\sfp$). Our algorithm seeks the best choice of
the weights in our bound, for which our bound is thus guaranteed to be
a lower bound for that of \cite{CortesMohriMunozMedina2019},
regardless of how the weights $\sfp$ are chosen in their bound.

The weighted-discrepancy minimization algorithm of
\cite{CortesMohri2014} is based on a two-stage minimization of
\eqref{eq:P-bound2} and in that sense is sub-optimal compared to an
algorithm seeking to minimize the bound of Theorem~\ref{th:qgen}.

\textbf{$\sfq$ chosen uniformly $\alpha$-weighted on $S$.}
\label{sec:alpha-reweighted}
Let $d\! = \!\dis(\sP, \sQ)$, $\h d$ and $\h d = \dis(\h \sQ, \h
\sP)$. Consider the following simple, and in general suboptimal,
choice of $\sfq$ as a distribution defined by:
\begin{align*}
\ov \sfq & = \frac{\alpha m}{m + n}
\qquad \sfq_i =
\begin{cases}
\frac{\ov \sfq}{m} = \frac{\alpha}{m + n} & \text{if } i \in [m];\\[.25cm]
\frac{1 - \ov \sfq}{n} = \frac{m (1 - \alpha) + n}{(m + n) n} & \text{otherwise},
\end{cases}
\end{align*}
where $\alpha = \Psi(1 - d)$ for some non-decreasing function $\Psi$ 
with $\Psi(0) = 0$ and $\Psi(1) = 1$. We will compare the right-hand
side of the bound of Theorem~\ref{th:qgen}, which we denote by $B$,
with its right-hand side $B_0$ for $\sfq$ chosen to be uniform over
$S'$ corresponding to supervised learning on just $S'$:
\[
B_0 = \cL(\h \sP, h) + 2 \Rad_n(\ell \circ \sH) + \sqrt{\frac{\log \frac{1}{\delta}}{2n}}.
\]
We now show that under some assumptions, we have $B - B_0 \leq
0$.  Thus, even for this sub-optimal choice of $\ov \sfq$, under those assumptions, the
guarantee of the theorem is then strictly more favorable than the one
for training on $S'$ only, uniformly over $h \in \sH$.

By definition of $\h d$, we can write:
\begin{align*}
\cL(\sfq, h)
& = \ov \sfq \cL(\h \sQ, h) + (1 - \ov \sfq) \cL(\h \sP, h)
\leq \ov \sfq \h d + \cL(\h \sP, h).
\end{align*}
By definition of the $\sfq$-Rademacher complexity 
and the sub-additivity of the supremum, the following inequality holds:
\[
\Rad_{\sfq}(\ell \circ \sH) \leq \ov \sfq \Rad_{m}(\ell \circ \sH) + (1 - \ov \sfq) \Rad_{n}(\ell \circ \sH).
\]
By definition of $\sfq$, we can write:
\begin{align*}
\norm{\sfq}^2_2 n
 = n \bracket*{m \paren*{\frac{\ov \sfq}{m}}^2 + n \paren*{\frac{1 -
  \ov \sfq}{n}}^2}
& = \frac{n}{m} \ov \sfq^2 + (1 - \ov \sfq)^2 \\
& = 1 - 2 \ov \sfq + \frac{m + n}{m} \ov \sfq^2\\
& = 1 - (2 - \alpha) \ov \sfq \leq 1 - \ov \sfq.
\end{align*}
Thus, using the inequality $\sqrt{1 - x} \leq 1 - \frac{x}{2}$, $x \leq 1$, we have:
\begin{align*}
B - B_0
& \leq 2 \ov \sfq \bracket*{\Rad_{m}(\ell \circ \sH) - \Rad_{n}(\ell \circ \sH) } 
+ \ov \sfq (d + \h d) + \bracket*{\sqrt{1 - \ov \sfq} - 1} \bracket*{\tfrac{\log
  \frac{1}{\delta}}{2n}}^{\frac{1}{2}}\\[-.25cm]
& \leq 2 \ov \sfq \bracket*{\Rad_{m}(\ell \circ \sH) 
- 
\Rad_{n}(\ell \circ \sH) } 
+ 
\ov \sfq (d + \h d)
- 
\ov \sfq \bracket*{\tfrac{\log
  \frac{1}{\delta}}{8n}}^{\frac{1}{2}}.
\end{align*}
Suppose we are in the regime of relatively small discrepancies and
that, given $n$, both the discrepancy and the empirical discrepancies
are upper bounded as follows: $\max \curl*{d, \ov d} <
\sqrt{\frac{\log 1/{\delta}}{32n}}$.
Assume also that for $m \gg n$ (which is the setting we are interested
in), we have $\Rad_{m}(\ell \circ \sH) - \Rad_{n}(\ell \circ \sH) \leq
0$. Then, the first term is non-positive and, regardless of the choice
of $\alpha < 1$, we have $B - B_0 \leq 0$. Thus, even for this
sub-optimal choice of $\ov \sfq$, under some assumptions, the
guarantee of the theorem is then strictly more favorable than the one
for training on $S'$ only, uniformly over $h \in \sH$.

Note that the assumption about the difference of Rademacher
complexities is natural. For example, for a kernel-based hypothesis
set $\sH$ with a normalized kernel such as the Gaussian kernel and the
norm of the weight vectors in the reproducing kernel Hilbert space
(RKHS) bounded by $\Lambda$, it is known that the following
inequalities hold: $\frac{1}{\sqrt{2}} \frac{ \Lambda}{\sqrt{m}} \leq
\Rad_m(\sH) \leq \frac{ \Lambda}{\sqrt{m}}$
\citep{MohriRostamizadehTalwalkar2018}. Thus, for $m > 2n$, we have
$\Rad_m(\sH) - \Rad_n(\sH) \leq \frac{ \Lambda}{\sqrt{m}} - \frac{
  \Lambda}{\sqrt{2n}} < 0$.

\subsection{Discrepancy estimation}
\label{app:disc-est}

First, note that if the $\sP$-drawn labeled sample at our disposal is
sufficiently large, we can reserve a sub-sample of size $n_1$ to train
a relatively accurate model $h_\sP$. Thus, we can subsequently reduce
$\sH$ to a ball $\sfB(h_\sP, r)$ of radius $r \sim
\frac{1}{\sqrt{n_1}}$. This helps us work with a finer (local)
discrepancy since the maximum in the definition is now taken over a
smaller set.

We do not have access to the discrepancy value $\dis(\sP, \sQ)$, which
defines $d_i$s. Instead, we can use the labeled samples from $\sQ$ and
$\sP$ to estimate it. Our estimate $\h d$ of the discrepancy is given
by
\begin{align*}
\h d
& = \max_{h \in \sH} \curl*{\frac{1}{n} \sum_{i = m + 1}^{m + n} \ell(h(x_i), y_i)
- \frac{1}{m} \sum_{i = 1}^m \ell(h(x_i), y_i) }.
\end{align*}
Thus, for a convex loss $\ell$, the optimization problems for
computing $\h d$ can be naturally cast as DC-programming problem,
which can be tackled using the DCA algorithm \citep{TaoAn1998} and
related methods already discussed for \sbest. For the squared loss,
the DCA algorithms is guaranteed to converge to a global optimum
\citep{TaoAn1998}.

By McDiarmid's inequality, with high probability,
$\abs{\dis(\sP, \sQ) - \h d}$ can be bounded by
$O(\sqrt{\frac{m + n}{mn}})$. More refined bounds such as relative
deviation bounds or Bernstein-type bounds provide more favorable
guarantee when the discrepancy is relatively small.  When $\sH$ is
chosen to be a small ball $\sfB(h_\sP, r)$, our estimate of the
discrepancy is further refined.

\subsection{Pseudocode of alternate minimization procedure}

\begin{figure*}[h]
\small
\begin{center}
  \vskip .15in
\fbox{\parbox{\textwidth}{
{\bf Input:} Samples $\{(x_1, y_1), \dots (x_{m+n}, y_{m+n})\}$, tolerance $\tau$, distribution $p_0$, max iterations $T$, hyperparameters $\lambda_\infty, \lambda_1, \lambda_2$, discrepancy estimate $\hat{d}$.
\begin{enumerate}   

\item Initialize $q_0$ to be the uniform distribution over $[m+n]$.

\item Initialize $h_0 = \argmin_{h \in H} \sum_{i=1}^{m+n} q_{0, i} \ell(h(x_i), y_i) + \lambda_\infty \|q_0\|_\infty \|h\|^2$.

\item For $t = 1, \dots T$,

\begin{itemize}

\item Set $\text{curr\_obj\_val} = \sum_{i=1}^{m} q_{t-1, i} \big(\ell(h_{t-1}(x_i), y_i) + \hat{d} \big) + \sum_{i=m+1}^{m+n} q_{t-1, i} \ell(h_{t-1}(x_i), y_i) + \lambda_\infty \|q_{t-1}\|_\infty \|h_{t-1}\|^2 + \lambda_1 \|q_{t-1}-p_0\|_1 + \lambda_2 \|q_{t-1}\|^2$.

\item Compute $q_{t} = \argmin_{q \in \Delta_{m+n}} \sum_{i=1}^{m} q_{i} \big(\ell(h_{t-1}(x_i), y_i) + \hat{d} \big) + \sum_{i=m+1}^{m+n} q_{i} \ell(h_{t-1}(x_i), y_i) + \lambda_\infty \|q\|_\infty \|h_{t-1}\|^2 + \lambda_1 \|q-p_0\|_1 + \lambda_2 \|q\|^2$.

\item Compute $h_{t} = \argmin_{h \in H} \sum_{i=1}^{m} q_{t, i} \big(\ell(h_{t-1}(x_i), y_i) + \hat{d} \big) + \sum_{i=m+1}^{m+n} q_{t, i} \ell(h_{t-1}(x_i), y_i) + \lambda_\infty \|q_t\|_\infty \|h\|^2$.

\item Set $\text{new\_obj\_val} = \sum_{i=1}^{m} q_{t, i} \big(\ell(h_{t}(x_i), y_i) + \hat{d} \big) + \sum_{i=m+1}^{m+n} q_{t, i} \ell(h_{t}(x_i), y_i)+ \lambda_\infty \|q_{t}\|_\infty \|h_{t}\|^2 + \lambda_1 \|q_{t}-p_0\|_1 + \lambda_2 \|q_{t}\|^2$.

\item If $|\text{curr\_obj\_val} - \text{new\_obj\_val}| \leq \tau$, return $q_t, h_t$
\end{itemize}
\item Print: \textit{AM did not converge in T iterations}. Return $q_T, h_T$.
\end{enumerate}
}}
\end{center}
%\vskip -.1in
\caption{Alternate minimization procedure for best effort adaptation.}
\label{ALG:am-labeled} 
%\vskip -.15in
\end{figure*}

\newpage
\section{Domain adaptation}
\label{app:da-unif}

\subsection{Theorems and proofs}
\label{app:adap-theorems}

Let $(\sfq, \sfq')$ denote the vector in $[0, 1]^{m + n}$ formed by
appending $\sfq'$ to $\sfq$.  The learning bound of
Theorem~\ref{th:da} can be extended to hold uniformly over all $\sfp$
in $[0, 1]^{[m]}$ and $(\sfq, \sfq')$ in
\[
\curl*{(\sfq, \sfq') \in [0,
    1]^m \times [0, 1]^n \colon 0 < \| (\sfq, \sfq') - \sfp^0 \|_1 <
  1},
\] where $\sfp^0$ is a reference (or ideal) reweighting choice
over the $(m + n)$ points.

\begin{restatable}{theorem}{qDAUniform}
\label{th:da-unif}
For any $\delta > 0$, with probability at least $1 - \delta$ over the
choice of a sample $S$ of size $m$ from $\sQ$ and a sample $S'$ of
size $n$ from $\sP$, the following holds for all $h \in \sH$, $\sfq
\in \curl*{\sfq \colon 0 \leq \| (\sfq, \sfq') - \sfp^0 \|_1 < 1}$ and
all $\sfp \in [0, 1]^m$:
\begin{align*}
\cL(\sP, h)
& \leq \sum_{i = 1}^{m} (\sfq_i + \sfp_i) \ell(h(x_i), y_i)
+ \dis(\sfq', \sfp)\\
& \quad + \dis \paren[\Big]{\bracket[\big]{1 - \norm{\sfq'}_1} \sP, \norm{\sfq}_1 \sQ}\\
& \quad + \dis((\sfq, \sfq'), \sfp^0)
+ 2 \Rad_{(\sfq, \sfq')}(\ell \circ \sH)
+ 5 \norm{(\sfq, \sfq')  - \sfp^0}_1 \\
& \quad + \bracket*{\norm{\sfq}_2 + 2 \norm{(\sfq, \sfq')  - \sfp^0}_1}
\bracket*{ \sqrt{\log \log_2 \tfrac{2}{1 - \norm{(\sfq, \sfq')  - \sfp^0}_1}}
  + \sqrt{\tfrac{\log \frac{2}{\delta}}{2}} }.
\end{align*}
\end{restatable}

\begin{proof}
  The proof follows immediately by applying inequality
  \eqref{eq:disc-ineq}, which holds for all $\sfp \in [0, 1]^m$, to
  the bound of Theorem~\ref{th:qgen-uniform}.
\end{proof}

\begin{restatable}{corollary}{qDAUniformSimplified}
\label{cor:da-unif-simplified}
For any $\delta > 0$, with probability at least $1 - \delta$ over the
choice of a sample $S$ of size $m$ from $\sQ$ and a sample $S'$ of
size $n$ from $\sP$, the following holds for all $h \in \sH$, $\sfq
\in \curl*{\sfq \colon 0 \leq \| (\sfq, \sfq') - \sfp^0 \|_1 < 1}$ and
all $\sfp \in [0, 1]^m$:
\begin{align*}
\cL(\sP, h)
& \leq \sum_{i = 1}^{m} (\sfq_i + \sfp_i) \ell(h(x_i), y_i)
+ \dis(\sfq', \sfp)\\
& \quad + \norm{\sfq}_1 \dis \paren*{\sP, \sQ}\\
& \quad + \dis((\sfq, \sfq'), \sfp^0)
+  2 \Rad_{(\sfq, \sfq')}(\ell \circ \sH)
+ 6 \norm{(\sfq, \sfq')  - \sfp^0}_1 \\
& \quad + \bracket*{\norm{\sfq}_2 + 2 \norm{(\sfq, \sfq')  - \sfp^0}_1}
\bracket*{ \sqrt{\log \log_2 \tfrac{2}{1 - \norm{(\sfq, \sfq')  - \sfp^0}_1}}
  + \sqrt{\tfrac{\log \frac{2}{\delta}}{2}} }.
\end{align*}
\end{restatable}

\begin{proof}
  The result follows Theorem~\ref{th:da-unif} and the application
  of the upper bound used in the proof of Corollary~\ref{th:qgen}.
\end{proof}

\subsection{Proof of Lemma~\ref{lemma:disc-upper-bound1}}
\label{app:disc-upper-bound1}

\DiscUpperBoundOne*

\begin{proof}
  For any $h_0$, using the definition of the squared loss,
  the following inequalities hold:
\begin{align*}
\dis(\h \sP, \h \sQ) 
& = \sup_{h \in \sH} \, \abs*{
\E_{(x, y) \sim \h \sP}[\ell(h(x), y)]
- \E_{(x, y) \sim \h \sQ}[\ell(h(x), y)]
} \\
& \leq \sup_{h \in \sH} \, \abs*{
\E_{(x, y) \sim \h \sP}[\ell(h(x), h_0(x))]
- \E_{(x, y) \sim \h \sQ}[\ell(h(x), h_0(x))]}\\
& + \sup_{h \in \sH} \, \bigg| \E_{(x, y) \sim \h \sP}[\ell(h(x), y)]
- \E_{(x, y) \sim \h \sP}[\ell(h(x), h_0(x))] \\
& \mspace{60mu} + \E_{(x, y) \sim \h \sQ}[\ell(h(x), h_0(x))] 
- \E_{(x, y) \sim \h \sQ}[\ell(h(x), y)] \bigg|\\
& = \ov \dis_{\sH \times \set{h_0}}(\h \sP, \h \sQ)\\
& \quad + 2 \sup_{h \in \sH} \, \abs*{ \E_{(x, y) \sim \h \sP} \bracket*{h(x) \paren*{y - h_0(x)}}
  - \E_{(x, y) \sim \h \sQ} \bracket*{h(x) \paren*{y - h_0(x)}}}
\tag{def. of squared loss}\\
& = \ov \dis_{\sH \times \set{h_0}}(\h \sP, \h \sQ)
+ 2 \delta_{\sH, h_0}(\h \sP, \h \sQ).
\tag{def. of local discrepancy}
\end{align*}
This completes the proof.
\end{proof}

\subsection{Proof of Lemma~\ref{lemma:disc-upper-bound2}}
\label{app:disc-upper-bound2}

\DiscUpperBoundTwo*

\begin{proof}
When the loss function $\ell$ is $\mu$-Lipschitz with respect to
its second argument, we can use the following upper bound:
\begin{align*}
\dis(\h \sP, \h \sQ) 
& = \sup_{h \in \sH} \, \abs*{
\E_{(x, y) \sim \h \sP}[\ell(h(x), y)]
- \E_{(x, y) \sim \h \sQ}[\ell(h(x), y)]
} \\
& \leq \sup_{h \in \sH} \, \abs*{
\E_{(x, y) \sim \h \sP}[\ell(h(x), h_0(x))]
- \E_{(x, y) \sim \h \sQ}[\ell(h(x), h_0(x))]}\\
& \quad + \sup_{h \in \sH} \, \bigg| \E_{(x, y) \sim \h \sP}[\ell(h(x), y)]
- \E_{(x, y) \sim \h \sP}[\ell(h(x), h_0(x))] \\
& \quad + \E_{(x, y) \sim \h \sQ}[\ell(h(x), h_0(x))] 
- \E_{(x, y) \sim \h \sQ}[\ell(h(x), y)] \bigg|\\
& \leq \ov \dis_{\sH \times \set{h_0}}(\h \sP, \h \sQ)
+ \mu \E_{(x, y) \sim \h \sP}\bracket*{\abs*{y - h_o(x)}}
+ \mu \E_{(x, y) \sim \h \sQ}\bracket*{\abs*{y - h_o(x)}}.
\tag{$\ell$ assumed $\mu$-Lipschitz}
\end{align*}
This completes the proof.
\end{proof}

\subsection{Sub-Gradients and estimation of unlabeled discrepancy terms}
\label{app:udisc}

Here, we first describe how to compute the sub-gradients of the
unlabeled weighted discrepancy term $\dis(\sfq', \sfp)$ that appears
in the optimization problem for domain adaptation \eqref{eq:daopt},
and similarly $\ov \dis((\sfq, \sfq'), \sfp^0)$, in the case of the
squared loss with linear functions. Next, we show how the same
analysis can be used to compute the empirical discrepancy term $\ov
\dis(\h \sP, \h \sQ)$, which provides an accurate estimate of $\ov d =
\ov \dis(\sP, \sQ)$.

\subsubsection{Sub-Gradients of unlabeled weighted discrepancy terms}

Let $\ell$ be the squared loss and let $\sH$ be the family of linear
functions defined by $\sH = \curl*{x \mapsto \bw \cdot \bPhi(x) \colon
  \norm{\bw}_2 \leq \Lambda}$, where $\bPhi$ is a feature mapping from
$\sX$ to $\Rset^k$.
We can analyze the unlabeled discrepancy term $\ov \dis(\sfq', \sfp)$
using an analysis similar to that of \cite{CortesMohri2014}.
By definition of the unlabeled discrepancy, we can write:
\begin{align*}
  \ov \dis(\sfq', \sfp)
& = \sup_{h, h' \in \sH} \curl*{\sum_{i = 1}^{n} \sfq'_i \ell(h(x_{m + i}), h'(x_{m + i}))
- \sum_{i = 1}^{m} \sfp_i \ell(h(x_{i}), h'(x_{i}))}\\
& = \sup_{\norm{\bw}_2 , \norm{\bw'}_2 \leq \Lambda} \curl*{\sum_{i = 1}^{n} \sfq'_i \bracket*{(\bw - \bw') \cdot \bPhi(x_{m + i})}^2
- \sum_{i = 1}^{m} \sfp_i \bracket*{(\bw - \bw') \cdot \bPhi(x_i)}^2}\\
& = \sup_{\norm{\bu}_2 \leq 2\Lambda} \curl*{\sum_{i = 1}^{n} \sfq'_i \bracket*{\bu \cdot \bPhi(x_{m + i})}^2
- \sum_{i = 1}^{m} \sfp_i \bracket*{\bu \cdot \bPhi(x_i)}^2}\\
& = \sup_{\norm{\bu}_2 \leq 2\Lambda} \curl*{\sum_{i = 1}^{n} \sfq'_i \bu^\top \bPhi(x_{m + i}) \bPhi(x_{m + i})^\top \bu
- \sum_{i = 1}^{m} \sfp_i \bu^\top \bPhi(x_i) \bPhi(x_i)^\top \bu}\\
& = \sup_{\norm{\bu}_2 \leq 2\Lambda} \curl*{\bu^\top \bracket*{\sum_{i = 1}^{n} \sfq'_i \bPhi(x_{m + i}) \bPhi(x_{m + i})^\top 
      - \sum_{i = 1}^{m} \sfp_i \bPhi(x_i) \bPhi(x_i)^\top} \bu}\\
& = 4\Lambda^2 \sup_{\norm{\bu}_2 \leq 1} \bu^\top \bM(\sfq', \sfp) \bu\\
& = 4\Lambda^2 \max\curl*{0, \sup_{\norm{\bu}_2 = 1} \bu^\top \bM(\sfq', \sfp) \bu}\\
& = 4\Lambda^2 \max\curl*{0, \lambda_{\max} \paren*{\bM(\sfq', \sfp)}},
\end{align*}
where $\bM(\sfq', \sfp) = \sum_{i = 1}^{n} \sfq'_i \bPhi(x_{m + i})
\bPhi(x_{m + i})^\top - \sum_{i = 1}^{m} \sfp_i \bPhi(x_i)
\bPhi(x_i)^\top$ and where $\lambda_{\max} \paren*{\bM(\sfq', \sfp)}$
denotes the maximum eigenvalue of the symmetric matrix $\bM(\sfq',
\sfp)$. Thus, the unlabeled discrepancy $\ov \dis(\sfq', \sfp)$ can be
obtained from the maximum eigenvalue of a symmetric matrix that is an
affine function of $\sfq'$ and $\sfp$. Since $\lambda_{\max}$ is a
convex function and since composition with an affine function
preserves convexity, $\lambda_{\max} \paren*{\bM(\sfq', \sfp)}$ is a
convex function of $\sfq'$ and $\sfp$. Since the maximum of two convex
function is convex, $\max\curl*{0, \lambda_{\max} \paren*{\bM(\sfq',
    \sfp)}}$ is also convex.

Rewriting $\lambda_{\max} \paren*{\bM(\sfq', \sfp)}$ as
$\max_{\norm{\bu}_2 = 1} \bu^\top \bM(\sfq', \sfp) \bu$ helps
derive its sub-gradient 
using the sub-gradient calculation of the maximum of a set of
functions:
\[
\nabla_{(\sfq', \sfp)} \lambda_{\max} \paren*{\bM(\sfq', \sfp)}
= \begin{bmatrix}
  \bu^\top \bPhi(x_{m + 1}) \bPhi(x_{m +  1})^\top \bu\\
  \vdots\\
  \bu^\top \bPhi(x_{m + n}) \bPhi(x_{m +  n})^\top \bu\\
  -\bu^\top \bPhi(x_1) \bPhi(x_1)^\top \bu\\
  \vdots\\
  -\bu^\top \bPhi(x_m) \bPhi(x_m)^\top \bu
\end{bmatrix}
= \begin{bmatrix}
  \paren*{\bPhi(x_{m +  1}) \cdot \bu}^2\\
  \vdots\\
  \paren*{\bPhi(x_{m +  n}) \cdot \bu}^2\\
  -\paren*{\bPhi(x_1) \cdot \bu}^2\\
  \vdots\\
  - \paren*{\bPhi(x_m) \cdot \bu}^2
\end{bmatrix},
\]
where $\bu$ is the eigenvector corresponding to the maximum
eigenvalue of $\bM(\sfq', \sfp)$.
Alternatively, we can approximate the maximum eigenvalue via the
softmax expression
\begin{align*}
  f(\sfq', \sfp)
  = \frac{1}{\mu} \log \bracket*{\sum_{j = 1}^k e^{\mu \lambda_j\paren*{\bM(\sfq', \sfp)}}}
  = \frac{1}{\mu} \log \bracket*{\Tr \paren*{ e^{\mu \bM(\sfq', \sfp)}} },
\end{align*}
where $e^{\mu \bM(\sfq', \sfp)}$ denotes the matrix exponential of
$\mu \bM(\sfq', \sfp)$ and $\lambda_j\paren*{\bM(\sfq', \sfp)}$ the
$j$th eigenvalue of $\bM(\sfq', \sfp)$. The matrix exponential can be
computed in $O(k^3)$ time by computing the singular value
decomposition (SVD) of the matrix.
We have:
\[
\lambda_{\max} (\bM(\sfq', \sfp))
\leq f(\sfq', \sfp)
\leq \lambda_{\max} (\bM(\sfq', \sfp)) + \frac{\log k}{\mu}.
\]
Thus, for $\mu = \frac{\log k}{\e}$, $f(\sfq', \sfp)$ provides a
uniform $\e$-approximation of $\lambda_{\max} (\bM(\sfq', \sfp))$.
The gradient of $f(\sfq', \sfp)$ is given for all $j \in [n]$ and $i
\in [m]$ by
\begin{align*}
  & \nabla_{\sfq'_j} f(\sfq', \sfp)
  = \frac{\tri*{e^{\mu \bM(\sfq', \sfp)}, \bPhi(x_{m + j})\bPhi(x_{m + j})^\top}}{\Tr \paren*{ e^{\mu \bM(\sfq', \sfp)}} }
  = \frac{\bPhi(x_{m + j})^\top e^{\mu \bM(\sfq', \sfp)} \bPhi(x_{m + j})}{\Tr \paren*{ e^{\mu \bM(\sfq', \sfp)}} }\\
  & \nabla_{\sfp_i} f(\sfq', \sfp)
  = - \frac{\tri*{e^{\mu \bM(\sfq', \sfp)}, \bPhi(x_{i})\bPhi(x_{i})^\top}}{\Tr \paren*{ e^{\mu \bM(\sfq', \sfp)}} }
  = \frac{\bPhi(x_{i})^\top e^{\mu \bM(\sfq', \sfp)} \bPhi(x_{i})}{\Tr \paren*{ e^{\mu \bM(\sfq', \sfp)}} }.
\end{align*}
The sub-gradient of the unlabeled discrepancy term $\ov \dis((\sfq,
\sfq'), \sfp^0)$ or a smooth approximation can be derived in a
similar, using the same analysis as above.

\subsubsection{Estimation of unlabeled discrepancy terms}

The unlabeled discrepancy $\ov d = \ov \dis(\sP, \sQ)$ can be
accurately estimated from its empirical version $\ov \dis(\h \sP, \h
\sQ)$ \citep{MansourMohriRostamizadeh2009}.  In view of the analysis
of the previous section, we have
\begin{align*}
\ov \dis(\h \sP, \h \sQ)
& = 4 \Lambda^2 \lambda_{\max} \paren*{\bM(\h \sP, \h \sQ)}\\
& = 4 \Lambda^2 \lambda_{\max} \paren*{ \frac{1}{n} \sum_{i = 1}^{n} \bPhi(x_{m + i}) \bPhi(x_{m + i})^\top 
- \frac{1}{m} \sum_{i = 1}^{m} \bPhi(x_i) \bPhi(x_i)^\top}.
\end{align*}
Thus, this last expression can be used in place of $\ov d$ in the
optimization problem for domain adaptation.

\newpage
\section{Further details about experimental settings}
\label{app:exp-details}

In this section we provide further details on our experimental setup
starting with best effort adaptation.

\subsection{Best-Effort adaptation}
\label{sec:beff-exp-app}

Recall that in this setting we have labeled data from both source and
target, however the amount of labeled data from the source is much
larger.  We start by describing the baselines that we compare our
algorithms with. For the best-effort adaptation problem two natural
baselines are to learn a hypothesis solely on the target $\sP$, or
train solely on the source $\sQ$. A third baseline that we consider is
the $\alpha$-reweighted $q$ as discussed in
Section~\ref{sec:alpha-reweighted}. Note, $\alpha=1$ corresponds to
training on all the available data with a uniform weighting.

\subsubsection{Simulated data} We first consider a
simulated scenario where $n$ samples from the target distribution
$\sP$ are generated by first drawing the feature vector $x$
i.i.d.\ from a normal distribution with zero mean and spherical
covariance matrix, i.e, $N(0, I_{d \times d})$. Given $x$, a binary
label $y \in \set{-1, +1}$ is generated as $\text{sgn}(w_p \cdot x)$
for a randomly chosen unit vector $w_p \in \mathbb{R}^{d}$. For a
fixed $\eta \in (0.5, 1)$, $m=1\mathord{,}000$ i.i.d.\ samples from
the source distribution $\sQ$ are generated by first drawing
$(1-\eta)m$ examples from $N(0, I_{d \times d})$ and labeled according
to $\text{sgn}(w_q \cdot x)$ where $\|w_p - w_q\| \leq \epsilon$, for
a small value of $\epsilon$. Notice that when $\epsilon$ is small, the
$(1-\eta)m$ samples are highly relevant for learning the target
$\sP$. The remaining $\eta m$ examples from $\sQ$ are all set to a
fixed vector $u$ and are labeled as $+1$. These examples represent the
noise in $\sQ$ and as $\eta$ increases the presence of such examples
makes $\dis(\sP, \sQ)$ larger. In our experiments we set $d = 20,
\epsilon=0.01$, and vary $\eta \in \{0.05, 0.1, 0.15, 0.2\}$.

% See Appendix~\ref{sec:exp-details} for details on how the various
% parameters of the problem are set.

On the above adaptation problem we evaluate the performance of the
previously discussed baselines with our proposed $\sbest$ algorithm
implemented via the alternate minimization, \sbest-AM, and the
DC-programming algorithms, \sbest-DC, where the loss function
considered is the logistic loss and the hypothesis set is the set of
linear models with zero bias. For each value of $\eta$, the results
are averaged over $50$ independent runs using the data generation
process described above.

Figure \ref{fig:simulated-sbest-lr-app} shows the performance of the
different algorithms for various values of the noise level $\eta$ and
as the number of examples $n$ from the target increases. As can be
seen from the figure, both $\alpha$-reweighting and the baseline that
trains solely on $\sQ$ degrade significantly in performance as $\eta$
increases. This is due to the fact the $\alpha$-reweighting procedure
cannot distinguish between non-noisy and noisy data points within the
$m$ samples generated from $\sQ$. 

In Figure~\ref{fig:sub1-app}(Left) we plot the best $\alpha$ chosen by
the $\alpha$-reweighting procedure as a function of $n$. For reference
we also plot the amount of mass on the non-noisy points from $\sQ$,
i.e., $(1-\eta) \cdot m/(m+n)$. As can be seen from the figure, as $n$
increases the amount of mass selected over the source $\sQ$
decreases. Furthermore, as expected this decrease is sharper as the
amount of noise level increases. In particular, $\alpha$-reweighting
is not able to effectively use the non-noisy samples from $\sQ$.

On the other hand, both \sbest-AM and \sbest-DC are able to counter
the effect of the noise by generating $\sfq$-weightings that are
predominantly supported on the non-noisy samples. In
Figure~\ref{fig:sub1-app}(Right) we plot the amount of probability
mass that the alternate minimization and the DC-programming
implementations of $\sbest$ assign to the noisy data points.
% As a baseline we also plot the probability mass on the noisy points
% assigned by a uniform weighting.

As can be seen from the figure, the total probability mass decreases
with $n$ and is also decreasing with the noise levels. These results
also demonstrate that our algorithms that compute a good $q$-weighting
can do effective outlier detection since they lead to solutions that
assign much smaller mass to the noisy points.

\begin{figure*}[t]   
\centering
\minipage{0.35\textwidth}
  \includegraphics[width=\linewidth]{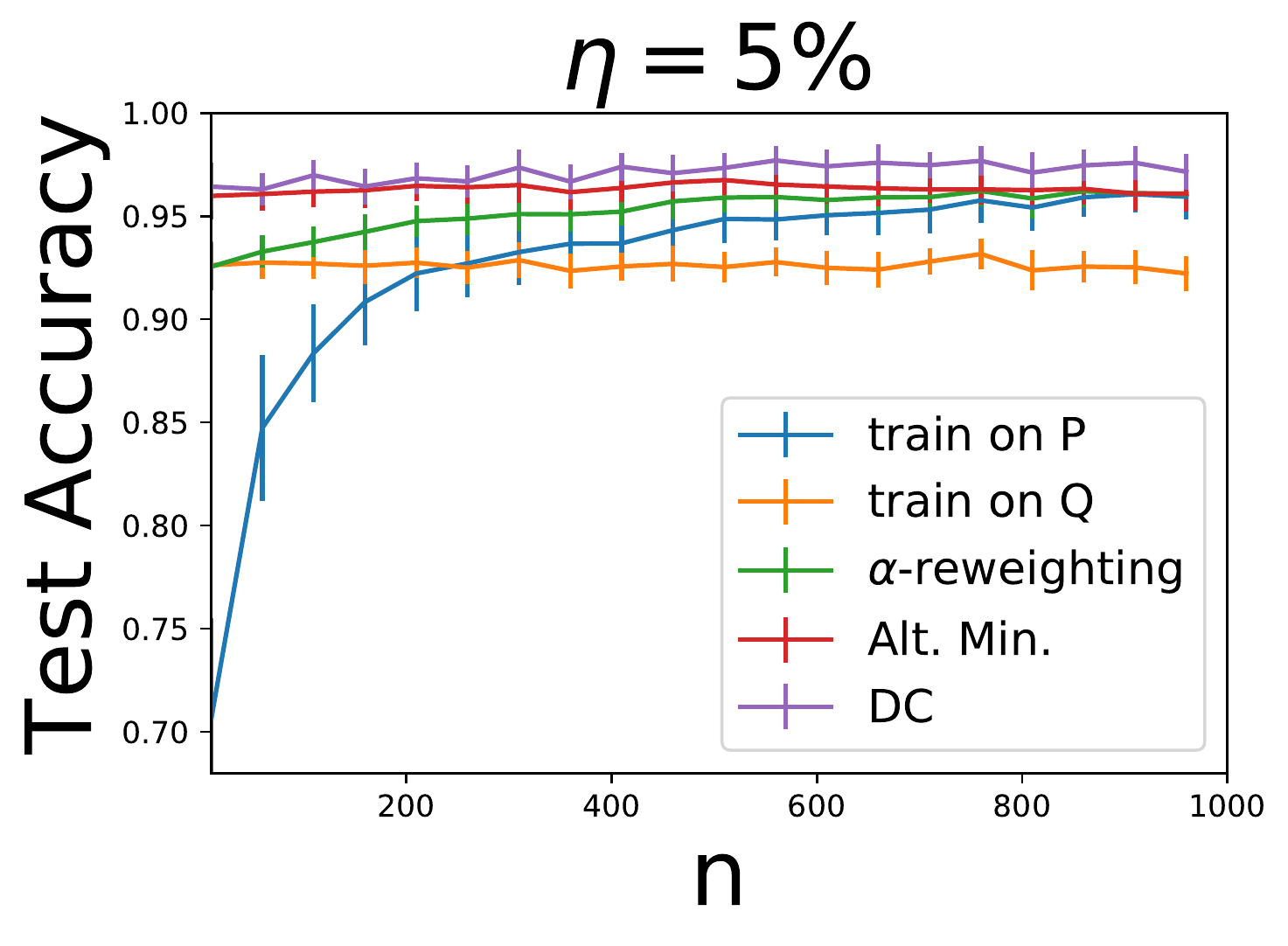}
\endminipage \hspace{5mm}
\minipage{0.35\textwidth}
  \includegraphics[width=\linewidth]{LR_gaussian_test_acc_eta_10.pdf}
\endminipage \hspace{5mm}
\minipage{0.35\textwidth}
  \includegraphics[width=\linewidth]{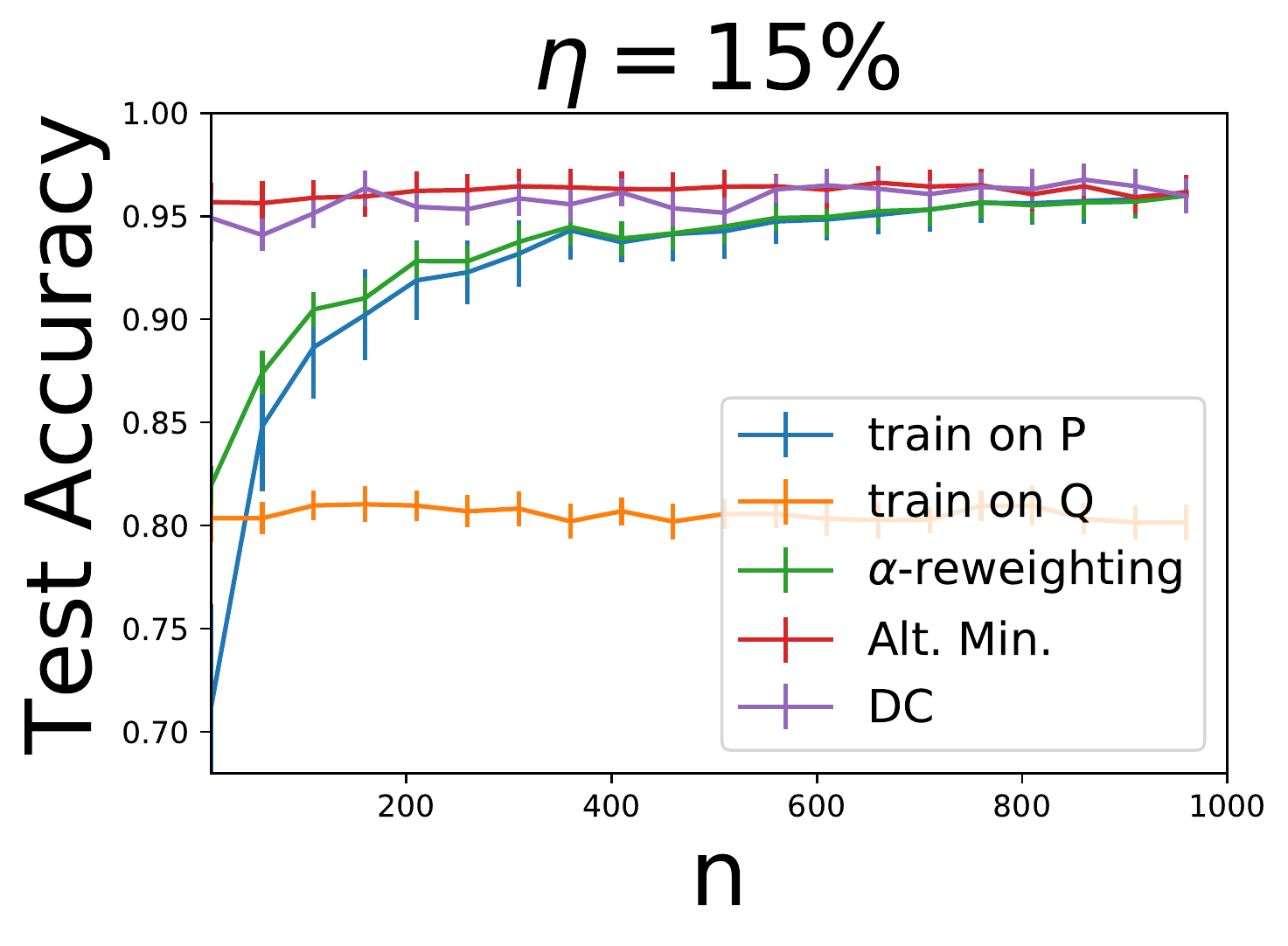}
\endminipage \hspace{5mm}
\minipage{0.35\textwidth}
  \includegraphics[width=\linewidth]{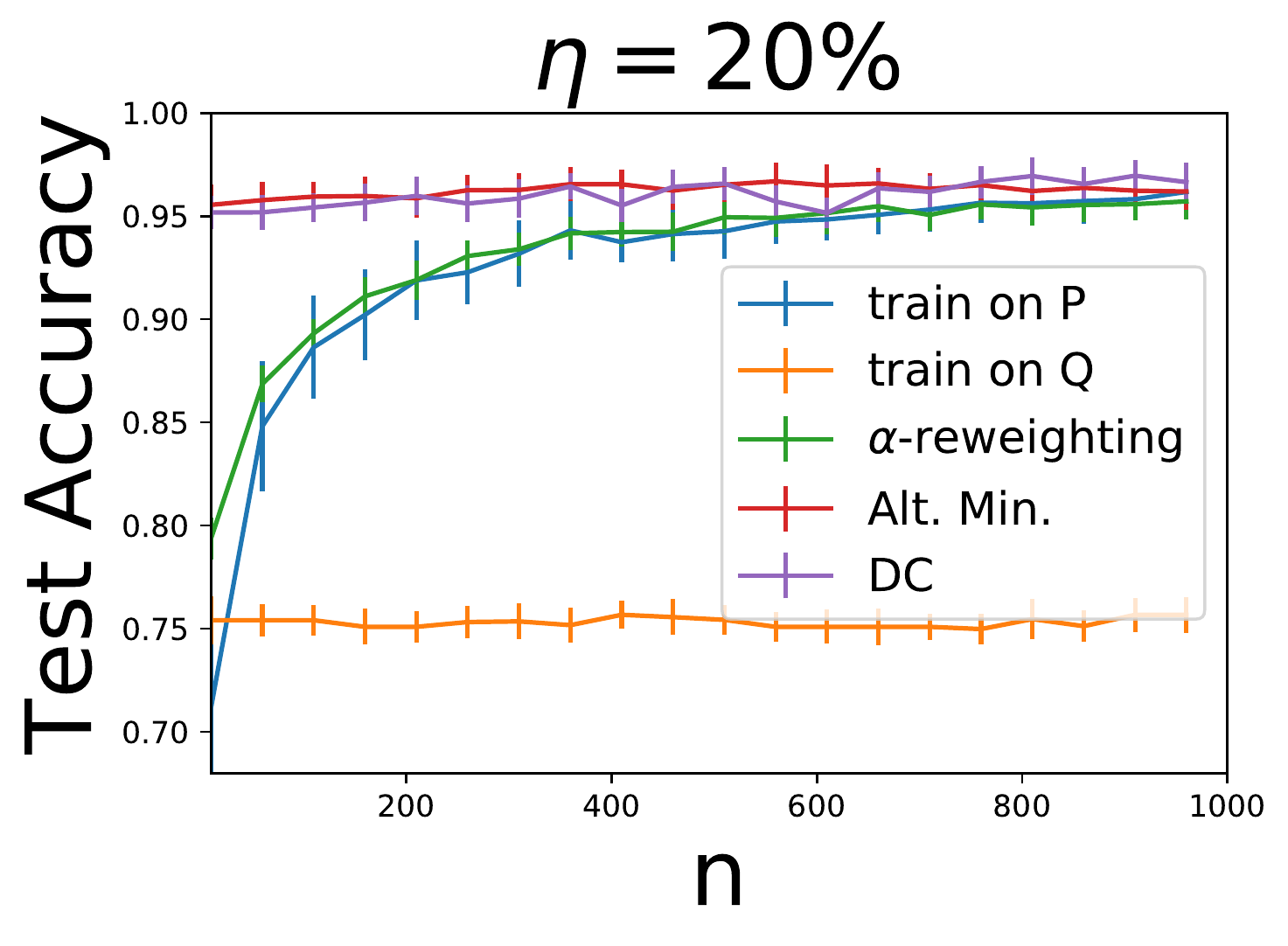}
\endminipage
\vspace{-0.25cm}
    \caption{Comparison of $\sbest$ against the baselines on simulated
      data in the classification setting. As the noise rate and
      therefore the discrepancy between $\sP$ and $\sQ$ increases the
      performance of the baselines degrades. In contrast, both
      the alternate minimization and the DC-programming algorithms
      effectively find a good $\sfq$-weighting and can adapt to the
      target.}
    \label{fig:simulated-sbest-lr-app}
\end{figure*}

\begin{figure}[!htb]   
\centering
\vskip .2in
\minipage{0.4\textwidth}
  \includegraphics[width=\linewidth]{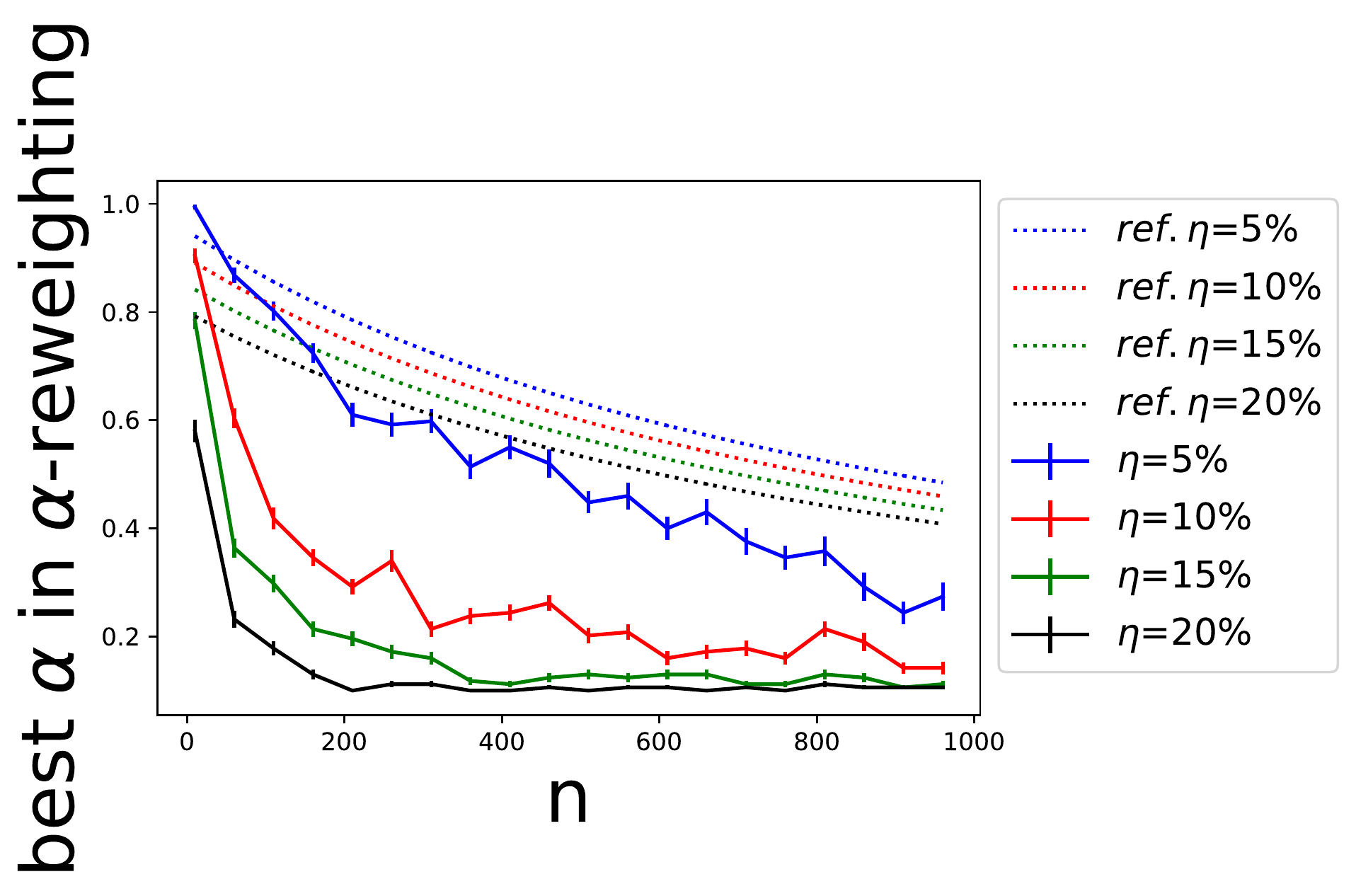}
\endminipage \hspace{1cm}
\minipage{0.4\textwidth}
  \includegraphics[width=\linewidth]{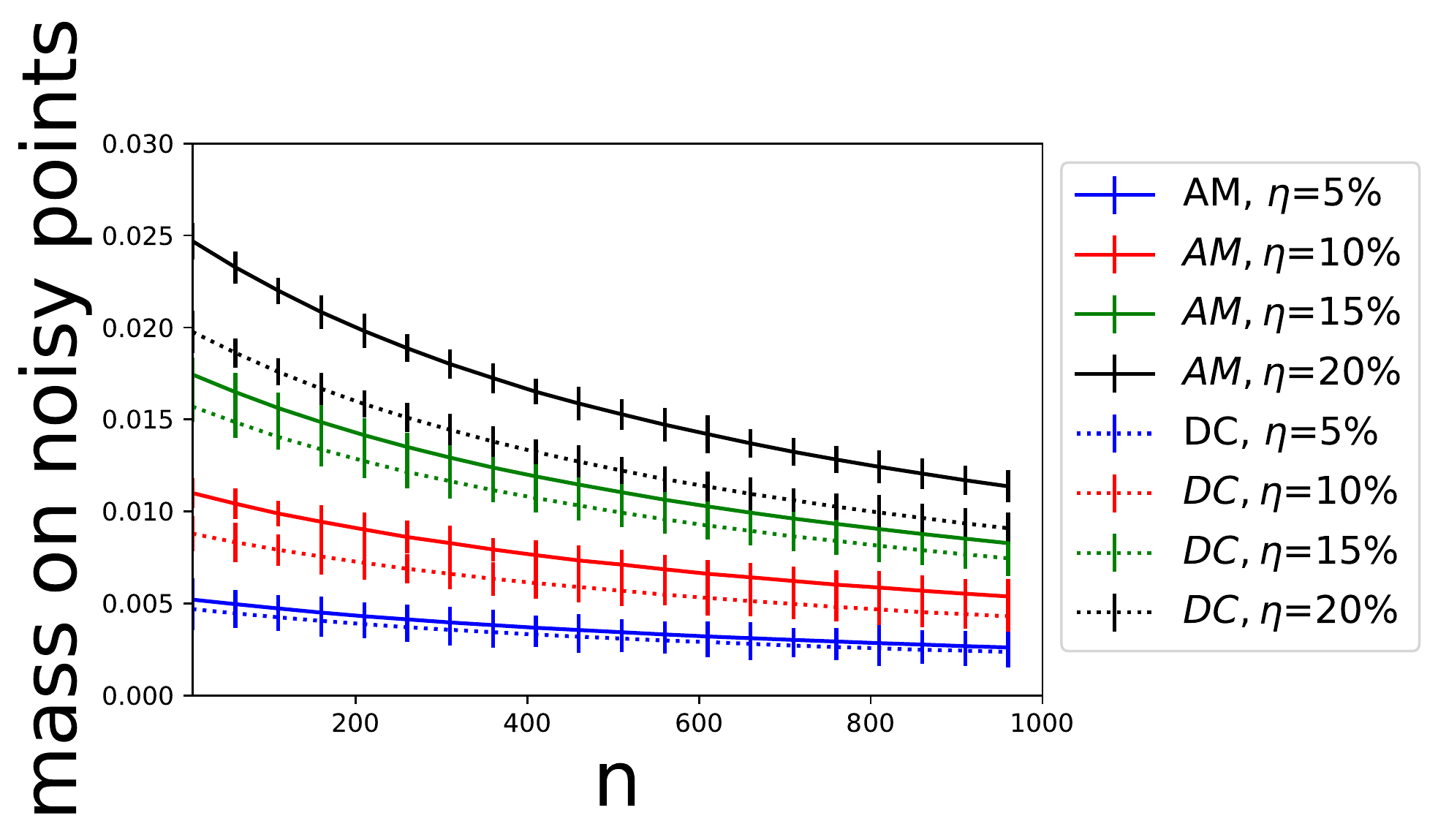}
\endminipage
\caption{\label{fig:sub1-app}(Left) Best $\alpha$ chosen by
  $\alpha$-reweighting as a function of $n$. (Right) Total probability
  mass assigned by $\sbest$ to the noisy points.}
\vskip .15in
\end{figure}

\subsection{Fine-tuning tasks}
\label{sec:finetuning-exp-app}

In this section we demonstrate the effectiveness of our proposed
algorithms for the purpose of fine-tuning pre-trained
representations. In the standard pre-training/fine-tuning paradigm
\citep{raffel2019exploring} a model is first pre-trained on a
generalist dataset (which is identified as coming from distribution
$\sQ$). Once a good representation is learned, the model is then
fine-tuned on a task specific dataset (generated from target
$\sP$). Two of the predominantly used fine-tuning approaches in the
literature are {\em last layer fine-tuning}
\citep{subramanian2018learning, kiros2015skip} and {\em full model
  fine-tuning} \citep{howard-ruder-2018-universal}. In the former
approach the representations obtained from the last layer of the
pre-trained model are used to train a simple model (often a linear
hypothesis) on the data coming from $\sP$. In our experiments we fix
the choice of the simple model to be a multi-class logistic regression
model. In the latter approach, the model when train on $\sP$, is
initialized from the pre-trained model and all the parameters of the
model are fine-tuned (via gradient descent) on the target distribution
$\sP$. In this section we explore the additional advantages of
combining data from both $\sP$ and $\sQ$ during the fine-tuning stage
via our proposed algorithms. There has been recent interest in
carefully combining various tasks/data for the purpose of fine-tuning
and avoid the phenomenon of ``negative transfer''
\citep{aribandi2021ext5}. Our proposed theoretical results present a
principled approach towards this.

To evaluate the effectiveness of our theory for this purpose, we
consider the CIFAR-10 vision dataset
\citep{krizhevsky2009learning}. The dataset consists of $50000$
training and $10000$ testing examples belonging to $10$ classes. We
form a pre-training task on data from $\sQ$, by combining all the data
belonging to classes: \{'airplane', 'automobile', 'bird', 'cat',
'deer', 'dog'\}. The fine-tuning task consists of data belonging to
classes: \{'frog', 'horse', 'ship', 'truck'\}. We consider both the
approaches of last layer fine-tuning and full-model fine-tuning and
compare the standard approach of fine-tuning only using data from
$\sP$ with our proposed algorithms. We use $60\%$ of the data from the
source for pre-training, and the remaining $40\%$ is used in
fine-tuning.

We split the fine-tuning data from $\sP$ randomly into a $70\%$
training set to be used in fine-tuning, $10\%$ for cross validation
and and the remaining $20\%$ to be used as a test set. The results are
reported over $5$ such random splits. We perform pre-training on a
standard ResNet-18 architecture \citep{he2016deep} by optimizing the
cross-entropy loss via the Adam optimizer. As can be seen in
Table~\ref{tbl:fine-tuning} both gapBoost and $\sbest$ that combine
data from $\sP$ and $\sQ$ lead to a classifier with better performance
for the downstream task, however, \sbest\ clearly outperforms
gapBoost.

The second dataset we consider is the \texttt{{Civil Comments}}
dataset \cite{pavlopoulos2020toxicity}. This dataset consists of text
comments in online forums and the goal is to predict whether a given
comment is toxic or not. Each data point is also labeled with {\em
  identity terms} that describes which subgroup the text in the
comment is related to. We create a subsample of the dataset where the
target consists of examples from the data points where the identity
terms is ``asian'' and the source is the remaining set of points. This
leads to $394,000$ points from the source and $20,000$ points from the
target. We create $5$ random splits of the data by randomly
partitioning the target data into $10,000$ examples for finetuning,
$2000$ for validation and $8000$ for testing. We perform pre-training
on a BERT-small model \citep{devlin2018bert} starting from the default
checkpoint as obtained from the standard tensorflow implementation of
the model.

\subsection{Domain adaptation}
\label{sec:domain-adaptation-results-app}

In this section we evaluate the effectiveness of our proposed $\bda$
objective for adaptation in settings where the target has very little
to no labeled data. In order to do this we consider multi-domain
sentiment analysis dataset of \citep{blitzer2007biographies} that has
been used in prior works on domain adaptation. The dataset consists of
text reviews associated with a star rating from $1$ to $5$ for various
different categories such as {\text{\sc books}}, {\text{\sc dvd}},
etc. We specifically consider four categories namely {{\text{\sc
      books}}, {\text{\sc dvd}}, {\text{\sc electronics}}, and
  {\text{\sc kitchen}}}. Inspired form the methodology adapted in
prior works \citep{MohriMunozMedina2012, CortesMohri2014}, for each
category, we form a regression task by converting the review text to a
$128$ dimensional vector and fitting a linear regression model to
predict the rating. In order to get the features we first combine all
the data from the four tasks and convert the raw text to a TF-IDF
representation using scikit-learn's feature extraction library
\citep{scikit-learn}. Following this, we compute the top $5000$ most
important features by using scikit-learn's feature selection library,
that in turn uses a chi-squared test to perform feature
selection. Finally, we project the obtained onto a $128$ dimensional
space via performing principal component analysis.

After feature extraction, for each task we fit a ridge regression
model in the $128$ dimensional space to predict the ratings. The
predictions of the model are then defined as the ground truth
regression labels. Following the above pre-processing we form 12
adaptation problems for each pair of distinct tasks: (TaskA, TaskB)
where TaskA, TaskB are in \{{\text{\sc books}}, {\text{\sc dvd}},
{\text{\sc electronics}}, {\text{\sc kitchen}}\}. In each case we form
the source domain ($\sQ$) by taking $500$ labeled samples from TaskA
and $200$ labeled examples from TaskB. The target ($\sP$) is formed by
taking $300$ unlabeled examples from TaskB. To our knowledge, there
exists no principled method for cross-validation in fully unsupervised
domain adaptation.  Thus, in our adaptation experiments, we used a
small labeled validation set of size $50$ to determine the parameters
for all the algorithms. This is consistent with experimental results
reported in prior work (e.g., \citep{CortesMohri2014}).

We compare our $\bda$ algorithm with the discrepancy minimization (DM)
algorithm of \cite{CortesMohri2014}, and the (GDM) algorithm,
\citep{CortesMohriMunozMedina2019}, which is a state of the art
adaptation algorithm for regression problems. We also compare with the
popular Kernel Mean Matching (KMM) algorithm,
\citep{HuangSmolaGrettonBorgwardtScholkopf2006}, for domain
adaptation. %For each problem, We report in Table \ref{tbl:sentiment}
the results averaged over $10$ independent source and target splits,
where we normalize the mean squared error (MSE) of $\bda$ to be $1.0$
and present the relative MSE achieved by the other methods. The
results show that in most adaptation problems, $\bda$ outperforms
(boldface) or ties with (italics) existing methods.

\subsubsection{Domain adaptation -- covariate-shift} Here we perform
experiments for domain adaptation only under covariate shift and
compare the performance of our proposed {\bda} objective with previous
state of the art algorithms. We again consider the multi-domain
sentiment analysis dataset \citep{blitzer2007biographies} from the
previous section and in particular focus on the {\em books}
category. We use the same feature representation as before and define
the ground truth as $y = w^* \cdot x + \sigma^2$ where $w^*$ is
obtained by fitting a ridge regression classifier. We let the target
be the uniform distribution over the entire dataset. We define the
source as follows: for a fixed value of $\epsilon$, we pick a random
hyperplane $w$ and consider a mixture distribution with mixture weight
$0.99$ on the set $w \cdot x \ge \epsilon$ and the mixture weight of
$0.01$ on the set $w \cdot x < \epsilon$. The performance of {\bda} as
compared to DM and KMM is shown in Table
\ref{tbl:sentiment-app-covariate}. As can be seen our proposed
algorithm either matches or outperforms current algorithms.

\begin{table}[t]
\caption{MSE achieved by $\bda$ as compared to DM and KMM on the
  covariate shift task for various values of $\epsilon$.}
\vskip -0.15in
\begin{center}
\resizebox{\columnwidth}{!}{
\begin{sc}
\begin{tabular}{lllllll}
Method & $\epsilon = 0$ &  $\epsilon = 0.2$ & $\epsilon = 0.4$ & $\epsilon = 0.6$ & $\epsilon = 0.8$ & $\epsilon = 1.0$\\
%\hrule
Train on $\sQ$ & $0.051 \pm 0.001$  & $0.06 \pm 0.001$ & $0.06 \pm 0.004$ & $0.07 \pm 0.006$ & $0.073 \pm 0.002$ & $0.073 \pm 0.005$\\
KMM & $0.05 \pm 1e-4$ & $0.05 \pm 1e-4$ & $0.05 \pm 3e-4$ & $0.06 \pm 1e-4$ & $0.06 \pm 1e-4$ & $0.07 \pm 2e-4$\\ 
DM & $0.02 \pm 0.005$ & $0.06 \pm 0.003$ & $0.05 \pm 0.003$ & $0.05 \pm 0.001$ & $0.06 \pm 0.005$ & $0.06 \pm 0.003$\\
$\bda$ & $0.01 \pm 0.006$ & $0.02 \pm 0.006$ & $0.027 \pm 0.005$ & $0.04 \pm 0.004$ & $0.04 \pm 0.007$ & $0.04 \pm 0.004$\\
%\hrule
\end{tabular}
\end{sc}
}
\end{center}
\vskip -0.1in
\label{tbl:sentiment-app-covariate}
\end{table}

\noindent \textbf{Hyperparameters for the algorithms}. 

For our proposed {\sbest} and {\sbda} algorithms the hyperparameters
$\lambda_\infty, \lambda_1, \lambda_2$ were chosen via
cross-validation in the union of the sets $\{1e-3, 1e-2, 1e-1\}$,
$\{0,1,2,\dots, 10\}$, and  $\{0, 1000, 2000, 10000, 50000, 100000\}$. The
$h$ optimization step of alternate minimization was performed using
sklearn's linear regression/logistic regression methods
\citep{scikit-learn}. During full layer fine-tuning on ResNet/BERT
models we use the Adam optimizer for the $h$ optimization step with
the default learning rates used for the CIFAR-10 dataset and the
BERT-small models.

For the $q$ optimization we used projected gradient descent and the
step size was chosen via cross validation in the range $\{1e-3, 1e-2,
1e-1\}$.

We re-implemented the gapBoost algorithm
\citep{WangMendezCaiEaton2019} in Python. Following the prescription
by the authors of gapBoost we set the parameter $\gamma = 1/n$ where
$n$ is the size of the target. We tune parameters $\rho_S, \rho_T$ in
the range $\{0.1, 0.2, \ldots ,1\}$ and the number of rounds of
boosting in the range $\{5,10,15,20\}$. We also re-implemented
baselines DM \citep{CortesMohri2014} and the GDM algorithm
\citep{CortesMohriMunozMedina2019}. These DM algorithm was implemented
via gradient descent and the second stage of the GDM algorithm was
implemented via alternate minimization. The learning rates in each
case searched in the range $\{1e-3, 1e-2, 1e-1\}$ and the
regularization parameters were searched in the range $\{1e-3, 1e-2,
1e-1, 0, 10, 100\}$. The radius parameter for GDM was searched in the
range $[0.01, 1]$ in steps of $0.01$.

To our knowledge, there exists no principled method for
cross-validation in fully unsupervised domain adaptation.  Thus, in
our unsupervised adaptation experiments, we used a small labeled
validation set of size $50$ to determine the parameters for all the
algorithms. This is consistent with experimental results reported in
prior work \citep{CortesMohri2014, CortesMohriMunozMedina2019}.

%\end{appendices}

\end{document}